%% file: main.tex
\documentclass{article} % For LaTeX2e
\usepackage{iclr2023_conference,times}
\usepackage{hyperref}
\usepackage{url}
\input{headers}

\doparttoc % Tell to minitoc to generate a toc for the parts
\faketableofcontents %

\title{Information-Theoretic Analysis of Unsupervised Domain Adaptation}

\author{Ziqiao Wang \& Yongyi Mao 
\\
University of Ottawa\\
\texttt{\{zwang286,ymao\}@uottawa.ca} 
}

\iclrfinalcopy % Uncomment for camera-ready version, but NOT for submission.
\date{}
\begin{document}

\maketitle

\begin{abstract}
This paper uses information-theoretic tools to analyze the generalization error in unsupervised domain adaptation (UDA). 
 We present novel upper bounds for two notions of generalization errors. The first notion measures the gap between the population risk in the target domain and that in the source domain, and the second measures the gap between the population risk in the target domain and the empirical risk in the source domain. While our bounds for the first kind of error are in line with the traditional analysis and give similar insights, our bounds on the second kind of error are algorithm-dependent, which also provide insights into algorithm designs. Specifically, we present two simple techniques for improving generalization in UDA and validate them experimentally.

%In particular, we first show that minimizing the KL divergence between the source domain and target domain is equivalent to minimizing many other discrepancy measures. Second, we prove some novel algorithm dependent bounds of UDA based on information-theoretic tools. Inspired by the theoretical results, we design some simple regularization techniques to further improve the performance of KL guided domain adaptation algorithm.
\end{abstract}

\input{Intro}
\input{Related}
\input{Preliminary}
\input{UnexpectedError}
\input{ExpectedError}

\input{Application}
\input{Experiments}

\section{Conclusion}
\label{sec:sums}

Despite that the numerous learning techniques have been developed for domain adaptation, significant room exists for more in-depth theoretical understanding and more principled design of learning algorithms. This paper presents the information-theoretic analysis for unsupervised domain adaptation, where we query two notions of the generalization errors in this context and present novel learning bounds. Some of these bounds recover the previous KL-based bounds under different conditions and confirm the insights in the learning algorithms that align the source and target distributions in the representation space. Our other bounds are algorithm-dependent, better exploiting the unlabelled target data, which have inspired novel and yet simple schemes for the design of learning algorithms. We demonstrate the effectiveness of these schemes  on standard benchmark datasets.

% \subsubsection*{Author Contributions}
% If you'd like to, you may include  a section for author contributions as is done
% in many journals. This is optional and at the discretion of the authors.

\subsubsection*{Acknowledgments}
This work is supported partly by an NSERC Discovery grant and a National Research Council of Canada (NRC) Collaborative R\&D grant (AI4D-CORE-07). Ziqiao Wang is also supported in part by the NSERC CREATE program through the Interdisciplinary Math and Artificial Intelligence (INTER-MATH-AI) project.
% is supported in part by the Interdisciplinary Math and Artificial Intelligence (INTER-MATH-AI) project, which is awarded through the NSERC CREATE Program. 
The authors would like to thank the anonymous reviewers for their careful reading and valuable suggestions.

\bibliography{ref}
\bibliographystyle{iclr2023_conference}

\newpage
\input{Appendix}

\end{document}

%% file: headers.tex
\usepackage{amsfonts, amssymb, amsmath, amsthm, bbm, graphicx, color, tikz, booktabs, multirow, bm, cases, mathtools,  mathrsfs, subcaption, pgfplots, csvsimple,algorithm,algorithmic,wrapfig}
\usetikzlibrary{intersections, pgfplots.fillbetween, shapes, arrows, positioning, decorations.markings}
\definecolor {processblue}{cmyk}{0.96,0,0,0}
\tikzstyle{int}=[draw, fill=blue!20, minimum size=2em]
\tikzstyle{init} = [pin edge={to-,thin,black}]
\usepackage{pgfplotstable}
\usepackage{pgfplots}
\pgfplotsset{compat=1.14}
\usepackage{hyperref, graphics}
\hypersetup{colorlinks=true,linkcolor=blue,filecolor=gray, urlcolor=blue, citecolor=blue}
\usepackage{natbib}
\newtheorem{defn}{Definition}[section]
\newtheorem{lem}{Lemma}[section]
\newtheorem{rem}{Remark}[section]
\newtheorem{assum}{Assumption}

\newtheorem{thm}{Theorem}[section]

\newtheorem{cor}{Corollary}[section]

\newcommand{\ex}[2]{{\ifx&#1& \mathbb{E} \else {\mathbb{E}_{#1}} \fi \left[#2\right]}}

\usepackage[toc,page,header]{appendix}
\usepackage{minitoc}

%% file: Intro.tex
\section{Introduction}
%Deep learning has achieved large success in many practical applications. One of the key assumption in many deep learning scenarios, either in theory or in application, is that the training data and the testing data come from the same distribution. However, such assumption does not always hold in practice, and the performance of trained models on the shifted distribution may behavior significantly  different with that on the original distribution. 
This paper focuses on the \textit{unsupervised domain adaptation (UDA)} task, where the learner is confronted with a source domain and a target domain and the algorithm is allowed to access to a labeled training sample from the source domain and an unlabeled training sample from the target domain. The goal is to find a predictor that performs well on the target domain.

A main obstacle in such a task is the discrepancy between the two domains.
%and if two domains do not have any similarity, then there is no hope for the model to generalize well. 
Some recent works \citep{ben2006analysis,ben2010theory,MansourMR09,zhao2019learning,zhang2019bridging,shen2018wasserstein,germain2020pac,acuna2021f,nguyen2022kl} have proposed various measures to quantify such discrepancy, either for the UDA setting or for the more general domain generalization tasks, and many learning algorithms are proposed. For example,  \citet{nguyen2022kl} uses a (reverse) KL divergence to measure the misalignment of the two domain distributions, and motivated by their generalization bound, they design an algorithm that penalizes the KL divergence between the marginal distributions of two domains in the representation space. Despite that this ``KL guided domain adaptation'' algorithm is  demonstrated to outperform many existing marginal alignment algorithms \citep{ganin2016domain,sun2016deep,shen2018wasserstein,li2018domain}, it is not clear whether  KL-based alignment of marginal distributions is adequate for UDA, and more fundamentally, what role the unlabelled target-domain sample should play in cross-domain generalization. Notably, most UDA algorithms are heuristically designed and intuitively justified. Moreover, most existing generalization bounds are algorithm-independent. Then there appears significant room for both deeper theoretical understanding and more principled algorithm design.

%Despite this KL guided algorithm significantly improves upon the previous methods, it remains unclear that the relation between KL divergence and the previous measures. In addition, many existing theoretical works focuses on the gap between the error on unseen target data (i.e. population risk on target domain) and that on unseen source data (i.e. population risk on source domain) for a given model, or a hypothesis. Although the bound of this gap is insightful (e.g., encouraging the alignment of distributions in representation space), the analysis for such gap becomes completely algorithm-independent, and could not characterize the change of hypothesis distribution when algorithm is allowed to access to unlabelled target domain data.

In this paper, we analyze the generalization ability of hypotheses and learning algorithms for UDA tasks using an information-theoretic framework developed in \citep{russo2016controlling,xu2017information}.  The foundation of our technique is the Donsker-Varadhan representation of KL divergence (see Lemma \ref{lem:DV-KL}).
%with the application of sub-gaussianity (see Assumption \ref{ass:subgaussian}). 
%With  the same spirit, we prove the relation between the KL-divergence measure and some existing measures. Further, to obtain the algorithm-dependent bound, we can directly follow the similar development of information-theoretic analysis for the UDA problem. To summarize, our contributions are
We present novel upper bounds for two notions of generalization errors. The first notion (``population-to-population (PP) generalization error'') measures the gap between the population risk in the target domain and that in the source domain {\em for a hypothesis}, and the second  (``expected empirical-to-population (EP) generalization error'') measures the gap between the population risk in the target domain and the empirical risk in the source domain {\em for a learning algorithm}. 
%While our bounds for the first kind of error are in line with the traditional analysis and give similar insights, our bounds on the second kind of error are algorithm-dependent and also inspire insights into algorithm designs. Specifically, we present two simple techniques for improving generalization in UDA and validate them experimentally.
%The specific contributions of this work are as follows.
%\begin{itemize}
    %\item %With Lemma \ref{lem:DV-KL} and Assumption \ref{ass:subgaussian}, 
    We show that the PP generalization error for all hypotheses are uniformly bounded by a quantity governed by the KL divergence between the two domain distributions, which, under bounded losses, recovers the the bound in \cite{nguyen2022kl}. We then show that this KL term upper-bounds some other measures including Total-Variation distance \citep{ben2006analysis}, Wasserstein distance \citep{shen2018wasserstein} and domain disagreement \citep{germain2020pac}. Thus, minimizing KL-divergence forces the minimization of other discrepancy measures as well. This, together with the ease of minimizing KL \citep{nguyen2022kl}, explains the effectiveness of the KL-guided alignment approach.
    %The significant improvement of minimizing KL is mainly due to, as also pointed out in \cite{nguyen2022kl}, the fact that the optimization process is more stable.
    %\item 
    For expected EP generalization error, we develop several algorithm-dependent generalization bounds.
    %some novel information-theoretic bound for the expected generalization error between the population risk of target domain and the empirical risk of source domain. 
    %\item 
    These algorithm-dependent bounds further inspire the design of two new and yet simple strategies that can further boost the performance of the KL guided marginal alignment algorithms.
    %\item 
    Experiments are performed %on standard benchmarks
    to verify the effectiveness of these strategies. 
    % \vspace{-2mm}
%\end{itemize}

%% file: Related.tex
\section{Related Work}
% \vspace{-1mm}
\paragraph{Domain Adaptation} 
Many domain adaptation generalization bounds have been developed \citep{ben2006analysis,ben2010theory,david2010impossibility,MansourMR09,shen2018wasserstein,zhang2019bridging,germain2020pac,acuna2021f}, and various discrepancy measures are introduced to derive these bounds including
%the reduction of 
total variation \citep{ben2006analysis,ben2010theory,david2010impossibility,MansourMR09}, Wasserstein distance \citep{shen2018wasserstein}, domain disagreement \citep{germain2020pac} and so on. In particular, bounds based on ${\cal H}\Delta {\cal H}$ in \cite{ben2010theory} are restricted to a binary classification setting and assume a deterministic labeling function. Furthermore, \cite{ben2010theory} also assumes the loss is the $L_1$ distance between the predicted label and true label (which is bounded). Our bounds work for the general supervised learning problems with any labelling mechanism (e.g., stochastic labelling), and we  do not require the specific choice of the loss (even unbounded).
% Motivated by the classic $f$-divergence, \citet{acuna2021f} proposed a discrepancy measure called $\mathrm{D}_{\mathcal{H}^\phi}$-discrepancy. Since KL divergence belongs to the family of $f$-divergences (e.g., choosing $x\log x$ as the Fenchel conjugate function) and both \cite{acuna2021f} and our work invoke the variational representation of the divergence, it seems our work (in Section~\ref{sec:unexpected-bound}) is related to theirs. However, the variational characterization of $f$-divergence used in \cite{acuna2021f} is based on the result of \cite{nguyen2010estimating}, and the Donsker-Varadhan representation of KL divergence (see Lemma~\ref{lem:DV-KL}) used in this paper cannot be directly recovered from their variational characterization \cite{jiao2017dependence,agrawal2020optimal}. Indeed, simply choosing $x\log x$ as the conjugate function will lead to a weaker bound than Lemma~\ref{lem:DV-KL}. Thus, our results (in Section~\ref{sec:unexpected-bound}) cannot be directly recovered from the results in \cite{acuna2021f}.
Recently, \cite{shui2020beyond} proposed generalization bounds using Jensen-Shannon (JS) divergence, which bear a relation to our Corollary \ref{cor-bound-symmetric-kl}. While other algorithm-dependent bounds have been proposed for different transfer learning settings (e.g., \cite{wang2019transfer}), they are not directly comparable to our own bounds.
% \cite{shui2020beyond} proposed generalization bounds based on Jensen-Shannon (JS) divergence, which are related to our Corollary \ref{cor-bound-symmetric-kl}. There also exist some algorithm-dependent bounds under other transfer learning setting (e.g., \cite{wang2019transfer}) that are not directly comparable to ours. 
% Most existing works give upper bounds for the ``PP'' error, while we give upper bounds for its absolute value, which also serves as a lower bound for generalization, highlighting some fundamental difficulty of the UDA learning task (see Corollary~\ref{cor:lower-bound}). 
For more details about the domain adaptation theory, we refer readers to \cite{redko2020survey} for a comprehensive survey. 
%From the algorithmic perspective of the domain adaptation, 
In addition, the most common methods for domain adaptation involve aligning the marginal distributions of the representations between the source and  target domains, for example, using an adversarial training mechanism \citep{ganin2016domain,shen2018wasserstein,acuna2021f} or aligning the first two moments of the representation distribution \citep{sun2016deep}. There are numerous other domain adaptation algorithms, and we refer readers to \citep{wilson2020survey,zhou2021domain,wang2021generalizing} for recent advances.
\vspace{-2mm}
% \textcolor{red}{Theory about domain adaptation, some well-known algorithms, for more details about the domain adaptation theory we refer readers to \cite{redko2020survey} for a completed survey }

\paragraph{Information-Theoretic Generalization Bounds}
Information-theoretic analysis is usually used to bound the expected generalization error of supervised learning, where the training and testing data come from the same distribution \citep{russo2016controlling,russo2019much,xu2017information,bu2019tightening,negrea2019information,steinke2020reasoning,rodriguez2021tighter}. Exploiting the chain rule of mutual information, these bounds are successfully applied to characterize the generalization ability of stochastic gradient based optimization algorithms \citep{pensia2018generalization,negrea2019information,haghifam2020sharpened,wang2021generalization,neu2021information,wang2022generalization,wang2022two}. Recently, this framework has also been used in other learning settings including meta-learning  \citep{jose2021information,jose2021transfer,rezazadeh2021conditional,chen2021generalization}, semi-supervised learning \citep{he2021information,aminian2022information} and transfer learning \citep{wu2020information,jose2021information,jose2021informationJS,masiha2021learning,bu2022characterizing}. 
In particular, \citep{wu2020information,jose2021informationJS} consider a different 
% transfer learning 
problem setup with ours. Specifically, their expected generalization error is the gap between the target population risk and a weighted empirical risk combining both the source and the target empirical risks, while our ``EP'' error is the gap between the target population risk and the source empirical risk. That is, we focus on the role of the {\em unlabelled} target data 
in cross-domain generalization when the source empirical risk is taken as a training objective, whereas their works assume the
existence of {\em labelled} target data and study their role in domain adaptation. 
%training objective function for the target domain data, which could be labelled, has already been known. 
% Additional discussion on the relationship between \citep{wu2020information,jose2021informationJS} and this work is given in Appendix after Theorem~\ref{thm:itb-uda-notarget}.

%In addition, bounds in \cite{wu2020information,jose2021informationJS} fail to characterize the dependence between $W$ and $S'_{X'}$. More precisely, the algorithm-dependent term in their bounds is $I(W;Z_i)$ or $I(W;S)$, while our algorithm-dependent term is $I^{X_j'}(W;Z_i)$ that directly depends on the unlabelled target data (see Theorem~\ref{thm:itb-uda-notarget} for more discussion in Appendix).
% \vspace{-2mm}
% \textcolor{red}{Information-theoretic bounds, xu's, single letters', cmi bound, sgld and sgd bound, wasserstein bound. Application on transfer learning and meta learning.}
% \vspace{-2mm}

%% file: Preliminary.tex
\section{Preliminary}
% \vspace{-2mm}
Unless otherwise noted,  a random variable will be denoted by a capitalized letter, and  its realization is denoted by the corresponding lower-case letter. Consider a prediction task with instance space  $\mathcal{Z}=\mathcal{X}\times\mathcal{Y}$,  where $\mathcal{X}$ and $\mathcal{Y}$ are the input space and the label (or output) space,  respectively. Let $\mathcal {F}$ be the hypothesis space of interest, in which each $f\in \mathcal {F}$ is a function or predictor mapping $\mathcal {X}$ to ${\cal Y}$. We assume that each hypothesis $f\in \mathcal {F}$ is parameterized by some weight parameter $w$ in some space ${\cal W}$ and may write $f$ as $f_w$ as needed. 

%Let $Z,Z'\in\mathcal{Z}$ be two random variables drawn from unknown distributions $\mu$ and $\mu'$, respectively. 
Let $\mu$ and $\mu'$ be two distributions on $\mathcal {Z}$, unknown to the learner, %Normally, $\mu$ and $\mu'$ are not the same and we consider 
where $\mu$ characterizes the source domain and $\mu'$ characterizes the target domain. 
%For convenience, we
We may also write $\mu$ as $P_Z$ or $P_{XY}$ and $\mu'$ as $P_{Z'}$ or 
$P_{X'Y'}$, which defines random variables $Z=(X, Y)$ and $Z'=(X', Y')$, respectively. Let $S=\{Z_i\}_{i=1}^n
%\overset{\text{i.i.d}}{\sim}
\sim
\mu^{\otimes n}$ be a labeled source-domain sample and
$S'_{X'} = \{X'_j\}_{j=1}^m
%\overset{\text{i.i.d}}{\sim}
\sim P_{X'}^{\otimes m}$ be an unlabelled target-domain sample. The objective of UDA is to design an  algorithm $\mathcal{A}$ that takes $S$ and $S'_{X'}$ as the input and outputs a weight $W\in\mathcal{W}$, giving rise to a predictor $f_W\in \mathcal {F}$ that ``works well'' on the target domain.  Note that the algorithm $\mathcal {A}$ is characterized by a conditional distribution 
$P_{W|S,S'_{X'}}$.

%To make a prediction, we select a predictor $f_w$ from the predictor space $\mathcal{F}=\{f_w:\mathcal{X}\rightarrow\mathcal{Y} | w\in\mathcal{W}\}$ according to the hypothesis $w$. Note that the mapping from $\mathcal{W}$ to $\mathcal{F}$ is a deterministic one-to-one mapping. 

%To be precise on the performance metric of UDA, let 
%$\ell:\mathcal{W}\times\mathcal{Z}
Let $\ell:\mathcal{Y}\times \mathcal {Y}
\rightarrow 
\mathbb{R}_0^+$ be a loss function. The population risk for each 
$w\in \mathcal {W}$ in the target domain is defined as
\[
R_{\mu'}(w)\triangleq \mathbb{E}_{Z'}{[\ell(f_w(X'),Y')]}
\]
and a good UDA algorithm hopes to return a weight $w$ that minimizes this risk. Since $\mu'$ is unknown, 
%this risk can not be measured or minimized. On the other hand, 
%instead of working with this risk, 
one often uses recourse to
the empirical risk in the source domain, defined as 
\[
R_S(w) \triangleq {\frac{1}{n}\sum_{i=1}^n\ell(f_w(X_i),Y_i)}. 
\]
%where $(X_i, Y_i)= Z_i$.
Generalization error in this setting measures how well the hypothesis returned from the algorithm generalizes from the source-domain training sample to the target-domain unknown distribution $\mu'$. Taking into account the stochastic nature of the algorithm $\mathcal {A}$, a natural notion of generalization error for UDA can be defined by 
\begin{align}
    \label{eq:expect-error}
{\mathrm{Err}}\triangleq \ex{W,S}{R_{\mu'}(W)-R_S(W)}=\ex{W,S,S'_{X'}}{R_{\mu'}(W)-R_S(W)},
\end{align}
where the expectation in the first expression is taken over the joint distribution of $(W,S)\sim P_{W|S}\times\mu^{\otimes n}$, and the expectation of the second expression is taken over the joint distribution of $(W,S,S'_{X'})\sim P_{W|S,S'_{X'}}\times\mu^{\otimes n}\times P_{X'}^{\otimes m}$.

There is another notion of generalization error, more traditional in the domain adaptation literature, defined as the gap between the population risk in the target domain and that in the source domain:
\begin{align}
    \widetilde{\mathrm{Err}}(w)\triangleq R_{\mu'}(w)-R_{\mu}(w). 
    \label{eq:unexpect-error}
\end{align}
where $R_{\mu}(w)\triangleq \mathbb{E}_Z{[\ell(f_w(X),Y)]}$. It is apparent that 
$\widetilde{\mathrm{Err}}(w)$ and $\mathrm{Err}$ are related by 
the following triangle inequality:
\[
|R_{\mu'}(w) - R_S(w)|\le |R_{\mu'}(w) - R_{\mu}(w)| + 
|R_{\mu}(w)-R_{S}(w)|. 
\]
where the second term on the right hand side is the standard generalization error in the source domain, which can be bounded by classical learning-theoretic tools, e.g., Rademacher complexity \citep{bartlett2002rademacher}. Thus, bounding $\widetilde{\mathrm{Err}}(w)$ helps bounding 
$\mathrm{Err}$. 

%In the UDA problem, we also have an unlabeled training sample from target domain, $S'_{X'} = \{X'_j\}_{j=1}^m\overset{\text{i.i.d}}{\sim}P_{X'}^{\otimes m}$, where $P_{X'}$ is the marginal distribution of $X'$ on target domain. 

%Then the algorithm $\mathcal{A}$ takes the training sample $S$ and $S'_{X'}$ as the input, and output $W$ according to some mapping $P_{W|S,S'_{X'}}$. During training, $\mathcal{A}$ does not access to any information of $Y'$. Our goal is to find an $w$ that minimizes the population risk of target domain, which is defined as:
%\[
%R_{\mu'}(w)\triangleq %\mathbb{E}_{Z'}{[\ell(f_w(X'),Y')]}. 
%\]

% In the traditional domain adaptation theory, for any $w\in\mathcal{W}$, we are first interested in the generalization gap between the population risk of target domain, $R_{\mu'}(w)$, and that of source domain, $R_{\mu}(w)\triangleq \mathbb{E}_Z{[\ell(f_w(X),Y)]}$,
% \begin{align}
%     \widetilde{\mathrm{Err}}(w)\triangleq R_{\mu'}(w)-R_{\mu}(w). 
%     \label{eq:unexpect-error}
% \end{align}

%Then by using the classic learning theory tools, e.g., Rademacher complexity \cite{bartlett2002rademacher}, to characterize the gap between $R_{\mu}(w)$ and $R_S(w)$, the generalization error bound for domain adaptation is obtained.

This paper studies both notions of generalization error for UDA. Specifically, 
%In this paper, we not only discuss $\widetilde{\mathrm{Err}}(w)$. Indeed, 
starting from Section~\ref{sec:expected-bound}, we will mainly use information-theoretic tools to bound $\mathrm{Err}$ directly, without going through $\widetilde{\mathrm{Err}}(w)$. For the ease of reference, we refer to $\widetilde{\mathrm{Err}}(w)$ as the {\em population-to-population (PP) generalization error for $w$} and $\mathrm{Err}$ as the {\em expected empirical-to-population (EP) generalization error.}
%for the algorithm $\mathcal {A}$}. 

%and focus on the expected generalization error between $R_{\mu'}(w)$ and $R_S(w)$, which is defined as

%Some definitions are prerequisite in this paper,  we now present some uncommon notions and defer the common notions to Appendix.
The following definitions are useful.
\begin{defn}[Disintegrated Mutual Information]
Let $X$, $Y$ and $Z$ be random variables and $z$ be a realization of $Z$. The disintegrated mutual information of $X$ and $Y$ given $Z=z$ is $I^z(X;Y)\triangleq \mathrm{D_{KL}}(P_{X,Y|Z=z}||P_{X|Z=z}P_{Y|Z=z})$.
\end{defn}
Note that the conditional mutual information  $I(X;Y|Z)=\mathbb{E}_{Z}{I^Z(X;Y)}$.

% \begin{defn}[Wasserstein Distance]
% Let $d(\cdot,\cdot)$ be a metric and let $P$ and $Q$ be probability measures on $\mathcal{X}$. Denote  $\Gamma(P,Q)$ as the set of all couplings of $P$ and $Q$ (i.e. the set of all joint distributions  on $\mathcal {X} \times \mathcal {X}$ with two marginals being $P$ and $Q$), then the Wasserstein Distance of order one between $P$ and $Q$ is defined as $\mathbb{W}(P,Q)\triangleq\inf_{\gamma\in\Gamma(P,Q)}\int_{\mathcal{X}\times\mathcal{X}} d(x,x')d\gamma(x,x')$.
% \end{defn}

% \begin{defn}[Total Variation]
% The total variation between two probability measures  $P$ and $Q$ is $\mathrm{TV}(P,Q)\triangleq\sup_E\left|P(E)-Q(E)\right|$, where the supremum is over all measurable set $E$.
% \end{defn}

% Note that the total variation equals to the Wasserstein distance under the discrete metric (or Hamming distortion)  $d(x,x')=\mathbbm{1}(x\neq x')$ where $\mathbbm{1}$ is the indicator function.

\begin{defn}[Lautum Information {\citep{palomar2008lautum}}]
 The lautum information between $X$ and $Y$ is defined as $L(X;Y)\triangleq\mathrm{D_{KL}}(P_XP_Y||P_{XY})$.
\end{defn}

% The key quantity in most information-theoretic generalization bounds is the mutual information between algorithm's input and output.
% %or equivalently, the expected KL divergence between some posterior distribution and prior distribution. 
% Specifically, the core technique behind these bounds is the well-known Donsker-Varadhan representation of KL divergence \cite[Theorem~3.5]{polyanskiy2019lecture}.
% \begin{lem}[Donsker and Varadhan's variational formula]
% \label{lem:DV-KL}
% Let $Q$, $P$ be probability measures on $\Theta$, for any bounded measurable function $f:\Theta\rightarrow \mathbb{R}$, we have
% $
% \mathrm{D_{KL}}(Q||P) = \sup_{f} \ex{\theta\sim Q}{f(\theta)}-\log\ex{\theta\sim P}{\exp{f(\theta)}}.
% $
% \end{lem}
% \vspace{-3mm}

%% file: UnexpectedError.tex
\section{Upper Bounds for PP Generalization Error}
\label{sec:unexpected-bound}
% \vspace{-3mm}
We now present some upper bounds for $\widetilde{\mathrm{Err}}(w)$. The key techniques used in developing these bounds are the information-theoretic tools in the style of Lemma \ref{lem:DV-KL}. These bounds adopt certain KL divergence to measure the discrepancy between the source and target domains. Notably, some previously established bounds are recovered under weaker conditions.  
 Additionally, we demonstrate that under certain conditions, the KL-based bound is an upper bound of several other discrepancy measures and hence minimizing the KL divergence forces the minimization of these other measures.

We first list some common assumptions on the loss function, which we consider in this paper.
\begin{assum}[Boundedness]
\label{ass:bounded}
    $\ell(\cdot,\cdot)$ is bounded in $[0,M]$.
\end{assum}
\begin{assum}[Subgaussianity]
\label{ass:subgaussian}
    $\ell(f_w(X),Y)$ is $R$-subgaussian\footnote{A random variable $X$ is $R$-subgaussian if for any $\rho$, $\log {\mathbb E} \exp\left( \rho \left(
X- {\mathbb E}X
\right) \right) \le \rho^2R^2/2$.} under $\mu$ for any $w\in \mathcal{W}$.
\end{assum}
\begin{rem}
\label{rem:boundedness}
Note that Assumption \ref{ass:bounded} implies Assumption \ref{ass:subgaussian}, i.e.,  if $\ell(f_w(X),Y)$ is bounded in $[0,M]$, then it is also $M/2$-subgaussian. Thus, Assumption \ref{ass:subgaussian} is weaker than Assumption \ref{ass:bounded}.
\end{rem}
\begin{assum}[Lipschitzness]
\label{ass:lipshitz}
 $\ell(f_w(X),Y)$ is $\beta$-Lipschitz continuous in $\mathcal{Z}$ with respect to a metric $d$ on $\mathcal{Z}$ for any $w\in \mathcal{W}$, i.e., $|\ell(f_w(x_1),y_1)-\ell(f_w(x_2),y_2)|\leq \beta d(z_1,z_2)$ for some metric $d$ on ${\mathcal Z}$.
\end{assum}
\begin{rem}
Note that Assumption \ref{ass:bounded} implies Assumption \ref{ass:lipshitz}  when $d(z_1,z_2)=\mathbbm{1}_{z_1\neq z_2}$, i.e.,  if $\ell(f_w(X),Y)$ is bounded in $[0,M]$, then it is also $M$-Lipschitz under the discrete metric.
\end{rem}
\begin{assum}[Triangle and Symmetric]
\label{ass:triangle}
$\ell(\cdot,\cdot)$ satisfies the following:
    $\ell(y_1,y_2)=\ell(y_2,y_1)$ and 
    $
    \ell(y_1,y_2)\leq\ell(y_1,y_3)+\ell(y_3,y_2)~\textit{for any}~y_1,y_2,y_3\in\mathcal{Y}
    $.
\end{assum}
    
\subsection{Generalization Bounds via the Subgaussian Condition}

 The following generalization bound is established by combining Lemma \ref{lem:DV-KL} and Assumption \ref{ass:subgaussian}, and its corresponding sample complexity bound is discussed in Appendix~\ref{sec:sample-complexity}.
%, a technique developed in \cite{xu2017information} for information-theoretic analysis of generalization. 

%In the information-theoretic analysis framework of generalization, the combination of Lemma \ref{lem:DV-KL} and Assumption \ref{ass:subgaussian} is widely used. We follow the similar spirit and obtain an upper bound for $\widetilde{\mathrm{Err}}(w)$ below.
   
        \begin{thm}
            \label{thm-bound-subgaussian}
            If Assumption \ref{ass:subgaussian} holds, then for any $w\in\mathcal{W}$,
             $
    \left|\widetilde{\mathrm{Err}}(w)\right| \leq \sqrt{2R^2\mathrm{D_{KL}}(\mu'||\mu)}
    $.
        \end{thm}
%   \vspace{-3mm}
   Notably this result can be turned into a generalization upper bound providing guidance to algorithm design, and at the same time it provides a lower bound of the generalization error, highlighting some fundamental difficulty of the learning task. To illustrate this, we present a corollary 
   %of Theorem \ref{thm-bound-subgaussian}, 
   while noting that similar development can also be applied to other bounds presented later in this paper.
   
   Consider that each $f_w$ 
   %in the model family 
   is expressed as the composition $g\circ h$, where $h$
is a function mapping ${\cal X}$ to a representation space ${\cal T}$  and $g$ is a function mapping ${\cal T}$ to ${\cal Y}$. For any given $h:{\cal X} \rightarrow {\cal T}$, denote by $\mu_{h}$ the distribution on ${\cal T}\times {\cal Y}$  obtained by pushing forward $\mu$ via $h$, that is, $
\mu_{h} (t, y)= \int \delta(t-h(x)) d\mu(x, y)$,
where $\delta$ is the Dirac measure on ${\cal T}$. Similarly, let $\mu'_{h}$ denote the distribution on ${\cal T}\times {\cal Y}$ obtained by pushing forward $\mu'$ via $h$. 
\begin{cor} 
\label{cor:lower-bound}
Suppose that $f_w=g\circ h$ and that Assumption \ref{ass:subgaussian} holds, then for any $w\in\mathcal{W}$,
\begin{align*}
    R_\mu(w)-
 \sqrt{2R^2\mathrm{D_{KL}(\mu'||\mu)}}
 \le
    R_{\mu'}(w) \leq R_\mu(w) + \sqrt{2R^2\mathrm{D_{KL}(\mu'_h||\mu_h)}}.
\end{align*}
\end{cor}   
In this result, the lower bound of $R_{\mu'}(w)$ indicates a fundamental difficulty in UDA learning in that, using the same predictor mapping $f_w$, there is no way for the population risk in the target domain to be lower than that of the source domain less than a constant which depends only on the domain difference. On the other hand, the upper bound suggests that it is possible to squeeze the gap between the two population risks by choosing an appropriate representation map $h$ - evidently such a map should be attempting to align $\mu'_h$ with $\mu_h$  or to align their respective proxies.

        It is also noteworthy that under Assumption \ref{ass:bounded} and due to Remark \ref{rem:boundedness}, Theorem \ref{thm-bound-subgaussian} implies 
        \begin{align}
        \label{ineq:cor-bound-subgaussian-bounded}
            \left|\widetilde{\mathrm{Err}}(w)\right| \leq {\frac{M}{\sqrt{2}}}\sqrt{\mathrm{D_{KL}}(P_{X'}||P_{X})+\mathrm{D_{KL}}(P_{Y'|X'}||P_{Y|X})}.
        \end{align}
    %     \begin{cor}
    %         \label{cor-bound-subgaussian-bounded}
    %         If Assumption \ref{ass:bounded} holds,
    %          $
    % \left|\widetilde{\mathrm{Err}}(w)\right| \leq {\frac{M}{\sqrt{2}}}\sqrt{\mathrm{D_{KL}}(P_{X'}||P_{X})+\mathrm{D_{KL}}(P_{Y'|X'}||P_{Y|X})}.
    % $
    %     \end{cor}
        Similarly applying this result in the representation space $\mathcal {T}$, we see that
        %Representing $X'$ and $X$ by $T'$ and $T$ in $\mathcal{T}$, respectively,
        Eq.~(\ref{ineq:cor-bound-subgaussian-bounded}) recovers the bound in Proposition 1 of \cite{nguyen2022kl}. Notice that unlike \cite{nguyen2022kl}, Theorem~\ref{thm-bound-subgaussian} ( or Eq.~(\ref{ineq:cor-bound-subgaussian-bounded})) does not require the loss to be the cross-entropy loss.
        
        Theorem~\ref{thm-bound-subgaussian} and \cite{nguyen2022kl} both use the KL divergence from source domain to target domain, $\mathrm{D_{KL}}(\mu'||\mu)$, and in fact, $\left|\widetilde{\mathrm{Err}}(w)\right|$ can also be upper bounded by $\mathrm{D_{KL}}(\mu||\mu')$. This can be done by invoking the subgaussianality of $\ell(f_w(X'),Y')$ (rather than $\ell(f_w(X),Y)$); for bounded loss, 
        the subgaussianality of $\ell(f_w(X'),Y')$
        is also satisfied. Then we obtain the following corollary.
        \begin{cor}
            \label{cor-bound-symmetric-kl}
            If Assumption \ref{ass:bounded} holds,
             $
    \left|\widetilde{\mathrm{Err}}(w)\right| \leq \frac{M}{\sqrt{2}}\sqrt{\min\{\mathrm{D_{KL}}(\mu||\mu'),\mathrm{D_{KL}}(\mu'||\mu)\}} \leq \frac{M}{2}\sqrt{{\mathrm{D_{KL}}(\mu||\mu')+\mathrm{D_{KL}}(\mu'||\mu)}}.
    $
        \end{cor}
        \begin{rem}
        In the second inequality of Corollary~\ref{cor-bound-symmetric-kl}, $\mathrm{D_{KL}}(\mu||\mu')+\mathrm{D_{KL}}(\mu'||\mu)$ is known as the symmetrized KL divergence, or
        Jeffrey’s divergence \citep{jeffreys1946invariant}, and in fact, 
        \cite{nguyen2022kl} penalizes this measure between the source and target distributions in the representation space.
        %the regularization term used in \cite{nguyen2022kl} is indeed the symmetrized KL divergence between the distributions of the source and target representations. 
        Notice that bounds in \cite{shui2020beyond} are based on the JS divergence. Since there is a sharp upper bound of the JS divergence based on Jeffrey’s divergence \citep{Crooks2008InequalitiesBT},  minimizing Jeffrey’s divergence (in the representation space) will simultaneously penalize the JS divergence.
        \end{rem}
        
        In UDA, since $Y'$ is completely unavailable to the algorithm $\mathcal{A}$, it is impossible to minimize the misalignment of conditional distributions, i.e. $\mathrm{D_{KL}}(P_{Y'|T'}||P_{Y|T})$ where $T$ and $T'$ are representations of source domain and target domain, respectively. 
        %without any additional information.
        A common method is to assign pseudo labels to target data based on a learned source classifier \citep{liang2020we}. However, it may also cause additional issues \citep{shen2022benefits}. For concreteness, suppose the trained model $Q$ can well approximate the real mapping between $X$ and $Y$ on source domain (i.e. $Q_{Y|T}=P_{Y|T}$), which is usually the training objective. Let $\hat{Y'}$ be the pseudo label of $T'$ generated by the trained model, i.e., $Q_{\hat{Y'}|T'}=Q_{Y|T}$. Let $Q_{T',\hat{Y'}}=P_{T'}Q_{\hat{Y'}|T'}$, then the following holds, 
\begin{align}
    \mathrm{D_{KL}}(P_{T',Y'}||P_{T,Y})=\mathbb{E}_{P_{T',Y'}}\log\frac{P_{T',Y'}Q_{T',\hat{Y'}}}{Q_{T',\hat{Y'}}P_{T,Y}}=\mathrm{D_{KL}}(P_{T'}||P_{T})+\mathrm{D_{KL}}(P_{Y'|T'}||Q_{\hat{Y'}|T'}).
    \label{eq:pseudo-label}
\end{align}
For a specific $t'$, if $P(Y'=y'|T'=t')\neq 0$ and $Q(\hat{Y'}=y'|T'=t')= 0$, then the second term in RHS of Eq. (\ref{eq:pseudo-label}), $\mathrm{D_{KL}}(P_{Y'|T'}||Q_{\hat{Y'}|T'})\rightarrow\infty$. In this case, even when the marginal distributions are perfectly aligned, the overall value of the upper bound is large. Thus, incorrect pseudo labels may even have negative impact on the target domain performance. 
%and we hope two supports, $\mathrm{Supp}(P_{Y'})$ and $\mathrm{Supp}(P_{\hat{Y'}})$, could largely overlap with each other for every target data.
        
        %  In UDA, since $Y'$ is completely unavailable to the algorithm $\mathcal{A}$, it is impossible to minimize the misalignment of conditional distributions, i.e. $\mathrm{D_{KL}}(P_{Y'|T'}||P_{Y|T})$, without any additional information. 
         In fact, the misalignment of the conditional distributions appears to be the main difficulty of UDA \citep{ben2006analysis,acuna2021f}.
        %  If we only want to align the marginal distribution of two domain, $P_{X'}$ and $P_{X}$, instead of aligning the joint distributions, $P_{\mu'}$ and $P_\mu$, 
          The next corollary suggests that this difficulty may be alleviated when the loss function satisfies the triangle property,  namely, Assumption \ref{ass:triangle}. It can be verified that this assumption is satisfied by the 0-1 loss \footnote{Some losses that only satisfy a general version of Assumption \ref{ass:triangle} are discussed in Appendix~\ref{sec:generalize-triangle}}; this assumption has also been considered in previous works \citep{MansourMR09,shen2018wasserstein}.
    \begin{thm}
            \label{cor-bound-subgaussian-triangle}
            If Assumption \ref{ass:triangle} holds and let $\ell(f_{w'}(X),f_{w}(X))$ be $R$-subgaussian for any $w,w'\in \mathcal{W}$. Then for any $w$,
             $
    \widetilde{\mathrm{Err}}(w) \leq \sqrt{2R^2\mathrm{D_{KL}}(P_{X'}||P_{X})}+\lambda^*$, where $\lambda^*=\min_{w\in\mathcal{W}}R_{\mu'}(w)+R_\mu(w)$.
        \end{thm}
        Here $\lambda^*$ measures the possibility of whether the domain adaptation algorithm will succeed  under the oracle knowledge of $\mu$ and
        $\mu'$. In particular, if the hypothesis space is large enough,  the minimizer $w^*$ for the ``joint population risk''
        $R_{\mu'}(w)+R_\mu(w)$ may give rise to 
   $R_{\mu'}(w^*)=R_\mu(w^*)=0$,  then we're likely to generalize well on the target domain.  Then the KL divergence 
   $\mathrm{D_{KL}}(P_{X'}||P_{X})$ between the two $\mathcal {X}$-marginals alone bounds the PP generalization error uniformly for all $w\in \mathcal{W}$.

   %Moreover, if source domain and target domain share the same labelling function, $f_{\mathrm{label}}$, and it belongs to $\mathcal{F}$, then $\lambda^*$ is zero.
    
    This theorem motivates the strategy of penalizing  $\mathrm{D_{KL}}(P_{T'}||P_{T})$ in the representation space for UDA. The next theorem suggests that such an approach also penalizes other notions of domain discrepancy, for example, the key quantity in the PAC-Bayes type of domain adaptation generalization bounds \citep{germain2020pac}, that is defined as
    % in \cite[Definition~1.]{germain2020pac}
    % which is a key quantity in the PAC-Bayes type of domain adaptation generalization bounds \citep{germain2020pac}:  
    \begin{align}
         \mathrm{dis}(P_X,P_{X'})\triangleq\left|\ex{W,W'}{\ex{X'}{\ell(f_{W}(X'),f_{W'}(X'))}}-\ex{W,W'}{\ex{X}{\ell(f_{W}(X),f_{W'}(X))}}\right|.
         \label{eq:domain-agree}
    \end{align}
    %This is another discrepancy measure and is the key quantity of a PAC-Bayes domain adaptation generalization bound in \cite{germain2020pac}.
    \begin{thm}
    \label{cor:kl-expected-disagreement} If $\ell(f_{w'}(X),f_w(X))$ is $R$-subgaussian for any $f_w,f_w'\in \mathcal{F}$, then
    $
       \mathrm{dis}(P_X,P_{X'})\leq\sqrt{2R^2\mathrm{D_{KL}}(P_{X'}||P_{X})}.
       $
    \end{thm}
    Note that unlike \cite{germain2020pac}, here we do not require the loss function to be the 0-1 loss.
    
\subsection{Generalization Bounds via the Lipschitz Condition}    
     %Wasserstein distance based generalization bound are often directly connected to, or even included in, the information-theoretic bounds \citep{wang2019information,rodriguez2021tighter}. 
     We now present such generalization bound for UDA under the Lipschitz continuity assumption of the loss function, where ${\mathbb W}(\cdot, \cdot)$ denotes the Wasserstein distance.
     
     %In particular, Lipschitz continuity is a common assumption to derive such a Wasserstein distance based bound, as given below.
        \begin{thm}
            \label{thm-bound-lipschitz} If Assumption \ref{ass:lipshitz} holds, then
             $
    \left|\widetilde{\mathrm{Err}}(w)\right| \leq \beta\mathbb{W}(\mu',\mu).
    $
        \end{thm}
    Theorem \ref{thm-bound-lipschitz}
    can be related to the KL-based bounds in the previous section when the Wasserstein distance is defined with respect to the discrete metric $d$. In this case and under bounded loss function, which is also Liptschitz continuous, Theorem \ref{thm-bound-lipschitz} follows. 
     On the other hand, Wasserstein distance is also equivalent to the total variation in this case,
        %, which is well studied in the domain adaptation theory 
        %\citep{ben2006analysis,ben2010theory,david2010impossibility,MansourMR09}
        % \citep{ben2006analysis},
         while the latter is connected to the KL divergence via Pinsker's inequality \cite[Theorem~6.5]{polyanskiy2019lecture} and the Bretagnolle-Huber inequality \cite[Lemma~2.1]{bretagnolle1979estimation}. 
        Thus, we arrive at the following result.
        %Further, since the Lipshitz condition is still not directly comparable to the subgaussian condition, we will add the boundness assumption. This is formally stated in the following corollary.
        \begin{cor}
            \label{cor-bound-total-variation-bounded} 
            If Assumption \ref{ass:bounded} holds holds and let $d$ be the discrete metric, then
             \[
                 \left|\widetilde{\mathrm{Err}}(w)\right| \leq M\mathrm{TV}(\mu',\mu)\leq M\sqrt{\min\left\{\frac{1}{2}\mathrm{D_{KL}}(\mu'||\mu),1-e^{-\mathrm{D_{KL}}(\mu'||\mu)}\right\}}.
             \]
        \end{cor}
    %      \begin{cor}
    %         \label{cor-bound-total-variation-bounded}
    %         If Assumption \ref{ass:bounded} holds and $d$ is a discrete metric, then
    %          $
    % \left|\widetilde{\mathrm{Err}}(w)\right| \leq M\mathrm{TV}(\mu',\mu)\leq M\sqrt{\min\left\{\frac{1}{2}\mathrm{D_{KL}}(\mu'||\mu),1-e^{-\mathrm{D_{KL}}(\mu'||\mu)}\right\}}.
    % $
    %     \end{cor}
        Note that results here are inspired by the work of \cite{rodriguez2021tighter}. Corollary \ref{cor-bound-total-variation-bounded} provides a tighter bound than the one in Eq.~(\ref{ineq:cor-bound-subgaussian-bounded}), as can be directly verified.
        %This corollary also indicates that KL divergence between two domains is also an upper bound of the total variation distance between the two domain distributions.
        
        Parallel to Theorem \ref{cor-bound-subgaussian-triangle}, if the loss function satisfies the triangle property, we may establish the bound below, which recovers a similar result in \cite[Theorem~1.]{shen2018wasserstein} but without restricting the task to be binary classification or requiring the loss to be the $L_1$ distance.
    \begin{thm}
            \label{thm-bound-lipshitz-triangle}
            If Assumption \ref{ass:triangle} holds and $\ell(f_w(X),f_{w'}(X))$ is $\beta$-Lipschitz in $\mathcal{X}$ for any $w,w'\in \mathcal{W}$, then for any $w\in\mathcal{W}$,
             $
    \widetilde{\mathrm{Err}}(w) \leq \beta\mathbb{W}(P_{X'},P_X)+\lambda^*$, where $\lambda^*=\min_{w\in\mathcal{W}}R_{\mu'}(w)+R_\mu(w)$.
        \end{thm}
        %Unlike the bound in \cite{shen2018wasserstein}, we do not require the classification tasks to be  binary in Theorem~\ref{thm-bound-lipshitz-triangle}, and the loss does not need to be the $L_1$ distance.
        
        % Let  $d$ be a discrete metric will also recover the bound in \cite{ben2006analysis}.
        These results justify the strategy of minimizing domain discrepancy in the representation space. 
        %This section may convey the following message. 
        Since the KL-based bounds upper-bound those based on other measures of domain differences, %(e.g. total variation distance, domain discrepancy etc), 
        penalizing the KL divergence will also penalize those other measures. This is practically advantageous since it is usually easier and more stable to minimize the KL divergence \citep{nguyen2022kl}.

        %that if we can minimize KL divergence of two domain distributions, then many other discrepancy measures will be simultaneously minimized. In practice, minimizing KL divergence is easier than some other measures, e.g. Wasserstein distance \cite{shen2018wasserstein}, and the optimization process is much more stable, as shown in \cite{nguyen2022kl}.

%% file: ExpectedError.tex
\section{Upper Bounds for EP Generalization Error and Applications}
\label{sec:expected-bound}

There are two limitations in the bounds on the PP generalization error developed so far and in the traditional 
analysis of UDA. First, such bounds are independent of $w$ and hence  algorithm-independent.  Second, although these  bounds may inspire strategies to exploit the unlabelled target sample, e.g., aligning the source and target distributions in the representation space, 
%its marginal distribution  with that of the source sample in the representation space,
they only provide very limited knowledge on the role that the unlabelled target sample plays. Inspired by the works of \cite{negrea2019information} and \cite{rodriguez2021random}, we derive upper bounds for the EP generalization error that take better advantage of the dependence of the algorithm's output on the unlabelled target data. Applications of these bounds in designing the learning algorithms are also presented.

\subsection{EP Generalization Bounds}
%In the information-theoretic framework, we're more interested in the expected generalization error, and under the domain adaption setting, this will be the expected gap between the population risk of target domain and the empirical risk of source domain, i.e. Eq. (\ref{eq:expect-error}). By following the similar development in many previous information-theoretic analysis works, we obtain Theorem~\ref{thm:ood-semi-bound-1}.
    \begin{thm}
    \label{thm:ood-semi-bound-1}
    Assume $\ell(f_W(X'),Y')$ is $R$-subgaussian under $P_{W,Z'|X_j'=x_j'}$ for any $x_j'\in\mathcal{X}$, then 
    \[
    \left|{\mathrm{Err}}\right|
    % \leq \frac{1}{nm}\sum_{i=1}^n\sum_{j=1}^m \sqrt{2\sigma^2(I(W;Z_i|X_j')+\mathrm{D_{KL}}(\mu||\mu'))}
    \leq \frac{1}{nm}\sum_{j=1}^{m}\sum_{i=1}^n\mathbb{E}_{X_j'}\sqrt{2R^2I^{X_j'}(W;Z_i)}+\sqrt{2R^2\mathrm{D_{KL}}(\mu||\mu')}.
    \]
    \end{thm}
    \begin{rem}
    \label{rem:vanishing}
    It is worth noting that the unlabelled target data contributes to the first term of the bound. Increasing the amount of source and target data will result in a reduction of the first term in the bound. Specifically, moving the expectation inside the square root function by Jensen's inequality and since $Z_i \perp\!\!\!\perp X_j'$, the equations $I(W;Z_i|X_j')=I(W;Z_i|X_j')+I(Z_i;X_j')=I(W;Z_i)+I(X_j';Z_i|W)$ hold by the chain rule. The term $I(W;Z_i)$ will vanish as $n\rightarrow\infty$ and the term $I(X_j';Z_i|W)$ will also vanish as $n,m\rightarrow\infty$.
    \end{rem}
    
    The theorem can be turned into a version that is more practically relevant, in which the KL term is replaced with their representation-space counter-part (following a similar argument used for deriving Corollary \ref{cor:lower-bound}). In addition, note that although larger sample sizes allow better estimation of that KL term, utilizing pseudo-labels for estimation may have a negative impact (as discussed in Section~\ref{sec:unexpected-bound}), which can be amplified by the larger sample size.

    %It is also worth mentioning that, from a practical perspective, the number of samples may have different impact on the different algorithms. For example, the second term (KL divergence) in our Theorem 5.1 can not be computed in the original space and we can only estimate it in the representation space. On the one hand, it seems that having more data will make the approximation (of KL between marginal distributions) more accurate. While on the other hand, some domain adaptation algorithms involve the pseudo labelling process, and assigning incorrect pseudo labels to the target data may even have negative impact on the target domain performance (as discussed in Section~\ref{sec:unexpected-bound}). In this case, having more target data will not improve the performance.

    % \textcolor{red}{decay rate}

    %When loss is bounded, then $\ell(f_W(X),Y)$ is also subgaussian, we have the following result.
    \begin{cor}
    \label{cor:bounded-mutual-lautum}
    Let Assumption~\ref{ass:bounded} hold. Then 
    \[
    \left|{\mathrm{Err}}\right|
    % \leq \frac{1}{nm}\sum_{i=1}^n\sum_{j=1}^m \sqrt{2\sigma^2(I(W;Z_i|X_j')+\mathrm{D_{KL}}(\mu||\mu'))}
    \leq \frac{M}{\sqrt{2}nm}\sum_{j=1}^{m}\sum_{i=1}^n\mathbb{E}_{X_j'}\sqrt{\min\left\{I^{X_j'}(W;Z_i),L^{X_j'}(W;Z_i)\right\}}+\frac{M}{\sqrt{2}}\sqrt{\min\left\{\mathrm{D_{KL}}(\mu||\mu'),\mathrm{D_{KL}}(\mu'||\mu)\right\}}.
    \]
    % where $L^{X_j'}(\cdot;\cdot)$ is the disintegrated version of Lautum information.
    \end{cor}
    %As Wasserstein distance based bound is in general included in the information-theoretic view, we also provide an tighter bound that based on Wasserstein distance.
    \begin{thm}
    \label{thm:wasserstein-bound-expected}
    Assume $\ell$ is Lipschitz for both $w\in\mathcal{W}$ and $z\in\mathcal{Z}$, i.e., $|\ell(f_w(x),y)-\ell(f_w(x'),y')|\leq \beta d_1(z,z')$ for all $z,z'\in\mathcal{Z}$ and $|\ell(f_w(x),y)-\ell(f_{w'}(x),y)|\leq \beta'd_2(w,w')$ for all $w,w'\in\mathcal{W}$, then
    \[
    \left|{\mathrm{Err}}\right|
    \leq \frac{\beta'}{nm}\sum_{j=1}^{m}\sum_{i=1}^n\mathbb{E}_{X_j',Z_i}{\mathbb{W}(P_{W|Z_i,X_j'},P_{W|X_j'})}+\beta\mathbb{W}(\mu,\mu').
    \] 
    \end{thm}
    This bound is tighter than the bound in Theorem~\ref{thm:ood-semi-bound-1}, as can be indicated by the following corollary. 
    \begin{cor}
    \label{cor:tv-bound-expected}
    Let Assumption~\ref{ass:bounded} hold. Then
    %if $d_1$ and $d_2$ are both discrete metrics, we have
    \begin{align*}
        \left|\widetilde{\mathrm{Err}}\right|&\leq \frac{M}{nm}\sum_{j=1}^{m}\sum_{i=1}^n\ex{X_j',Z_i}{\mathrm{TV}(P_{W|Z_i,X_j'},P_{W|X_j'})}+M\mathrm{TV}(\mu,\mu')\\
        &\leq \frac{1}{nm}\sum_{j=1}^{m}\sum_{i=1}^n\mathbb{E}_{X_j',Z_i}\sqrt{\frac{M^2}{2}\mathrm{D_{KL}}(P_{W|Z_i,X_j'}||P_{W|X_j'})}+\sqrt{\frac{M^2}{2}\mathrm{D_{KL}}(\mu||\mu')}.
    \end{align*}
    \end{cor}
    Notice that to recover Theorem~\ref{thm:ood-semi-bound-1} from Corollary~\ref{cor:tv-bound-expected} (under Assumption~\ref{ass:bounded}), we can use Jensen's inequality to move the expectation over $Z_i$ to inside the square root function.

%% file: Application.tex
    \subsection{Gradient Penalty as an Universal Regularizer}
    \label{sec:application-1}
    The algorithm-dependent bound in Theorem~\ref{thm:ood-semi-bound-1} tells us that one can reduce the EP error by limiting the disintegrated mutual information $I^{X_j'}(W;Z_i)$. In the stochastic gradient based optimization algorithms, this term can be controlled by penalizing the gradient norm. To see this, we now consider a ``noisy'' iterative algorithm for updating $W$, e.g., SGLD. At each time step $t$, let the labelled mini-batch from the source domain be $Z_{B_t}$, let the unlabelled mini-batch from the target domain be $X'_{B_t}$, and let $g(W_{t-1},Z_{B_t},X'_{B_t})$ be the gradient at time $t$. Thus, the updating rule of $W$ is $W_t=W_{t-1}-\eta_t g(W_{t-1},Z_{B_t},X'_{B_t})+N_t$ where $\eta_t$ is the learning rate and $N_t\sim\mathcal{N}(0,\sigma^2\mathrm{I}_d)$ is an isotropic Gaussian noise. Inspired by \cite{pensia2018generalization}, we have the following bound.
    \begin{thm}
    \label{thm:uda-gradient-bound}
    Let the total iteration number be $T$ and let $G_t=g(W_{t-1},Z_{B_t},X'_{B_t})$, then
    \[
    \left|{\mathrm{Err}}\right|\leq\sqrt{\frac{R^2}{n}\sum_{t=1}^T\frac{\eta_t^2}{\sigma^2_t}\ex{S'_{X'},W_{t-1},S}{\left|\left|G_t-\ex{Z_{B_t}}{G_t}\right|\right|^2}}+\sqrt{2R^2\mathrm{D_{KL}}(\mu||\mu')}.
    \]
    \end{thm}
    \begin{rem}
    Considering a noisy iterative algorithm here is merely for simplifying analysis. In fact, it is also possible to analyze the original iterative gradient optimization method without noise injected. For example, one can follow the same development in \citep{neu2021information,wang2022generalization} to analyze vanilla SGD. In that case, there will be some residual terms in the bound.
    % which are related to flatness of the found minima. 
    %Since \cite[Corollary~A.6.1.]{barrett2020implicit} shows that minimizing the gradient norm would help to find flat minima, those additional terms do not create conflicts to our argument on the noisy algorithm.
    \end{rem}
    
    Theorem~\ref{thm:uda-gradient-bound} hints that to reduce the generalization error, one can simply restrict the gradient norm at each step (so that $||G_t-\ex{Z_{B_t}}{G_t}||^2$ is reduced). This strategy will also restrict the distance between the final output $W_T$ and the initialization $W_0$, effectively shrinking the hypothesis space accessible by the algorithm. We also note that the importance of  gradient penalty has been theoretically justified in the supervised learning setting \citep{negrea2019information,haghifam2020sharpened,smith2020origin,rodriguez2021random,neu2021information,wang2022generalization,wang2022two}.

    Indeed, adding gradient penalty can be applied to any existing UDA algorithm and it is simple but effective in practice. Later on we will show that even when the algorithm $\mathcal{A}$ does not access to any target data, in which case $I(W;Z_i|X_j')$ reduces to $I(W;Z_i)$ and $g(W_{t-1},Z_{B_t},X_{B_t}')$ becomes $g(W_{t-1},Z_{B_t})$, minimizing the empirical loss of source domain sample while penalizing gradient norm will still improve the performance. Notice that gradient penalty has been used in standard supervised learning as a regularization technique \citep{geiping2022stochastic,jastrzebski2021catastrophic}. It is also used in Wasserstein distance based adversarial adaptation \citep{gulrajani2017improved,shen2018wasserstein}, and their motivation is to stabilize the training to avoid gradient vanishing problem.  Here we suggest, with strong theoretical justification, that gradient penalty is a universal technique for improving the generalization performance in UDA for any gradient-based learning method.

    Notably the bound in Theorem~\ref{thm:uda-gradient-bound} only depends on the size $n$ of labelled source sample and does not explicitly depend on $m$, the size of unlabelled target sample. With a more careful design, if we consider the mutual information as the expected KL divergence of a posterior and a prior, based on $I^{X_j'}(W;Z_i)$ in Theorem~\ref{thm:ood-semi-bound-1}, it is possible to create a target-data-dependent prior and derive a tighter bound based on some quantity similar to "gradient incoherence" in \cite{negrea2019information}. %As this will introduce additional complexity in practice, we leave this as a future study.

    \subsection{Controlling Label Information for KL Guided Marginal Alignment}
    \label{sec:application-2}
    
    Consider instances in the representation space, $Z=(T,Y)$ and $Z'=(T',Y)$. Theorem~\ref{thm:ood-semi-bound-1} also encourages us to align the distributions of two domains in the representation space, as argued earlier. Then the KL guided marginal alignment algorithm proposed in \cite{nguyen2022kl} can be invoked here. One may notice that Theorem~\ref{thm:ood-semi-bound-1} uses $\mathrm{D_{KL}}(\mu||\mu')$ while \cite{nguyen2022kl} uses $\mathrm{D_{KL}}(\mu'||\mu)$. As already discussed in Section~\ref{sec:unexpected-bound}, this inconsistency can be ignored when the loss is bounded (see Corollary~\ref{cor:bounded-mutual-lautum}).
    
    %As illustrated in Section~\ref{sec:unexpected-bound}, 
    Most domain adaptation algorithms aim to align the marginal distributions of two domains in the representation space. However, without accessing to $Y'$,  it remains unknown if an UDA algorithm will work well since we cannot guarantee that discrepancy between conditional distribution $P_{Y|T}$ and $P_{Y'|T'}$ won't  become too large when we align the marginals. 
    %This difficulty is evaluated by $\lambda^*$ if loss has triangle property, but this property is unrealistic in the most cases. 
    In \cite{nguyen2022kl}, the authors show that  $\mathrm{D_{KL}}(P_{Y'|T'}||P_{Y|T})$ can be upper-bounded by $\mathrm{D_{KL}}(P_{Y'|X'}||P_{Y|X})$, if $I(X;Y)=I(T;Y)$. The authors then argue that penalizing the KL divergence of the marginals 
    % distributions 
    is safe.

    We now argue that in practice the condition 
    $I(X;Y)=I(T;Y)$ can be difficult to satisfy if the cross-entropy loss is used to define the source-domain empirical risk.

    %One of the key assumptions in \cite{nguyen2022kl} is to assume $I(X;Y)=I(T;Y)$, we now elaborate on this sufficiency condition.
    
    By data processing inequality on $Y-X-T$, we know that $I(X;Y)\geq I(T;Y)=H(Y)-H(Y|T)$. Thus, to let $I(T;Y)$ reach its maximum, one must minimize $H(Y|T)$. 
    %\citet{nguyen2022kl} argues that this can be satisfied when we minimize the empirical risk on source domain. However, this is not always true when the model capacity is large. To understand this, 
    On the other hand, let $Q_{Y|T,W}$ be the predictive distribution of labels in the source domain generated by the classifier. The expected cross-entropy loss for each $Z_i$ in the representation space is then
    \[
    \ex{W,Z_i}{\ell(f_W(T_i),Y_i)} = \ex{Z_i}{\ex{W|Z_i}{-\log{Q_{Y_i|T_i,W}}}},
    \]
    % Indeed, this quantity is well studied in the supervised setting with the cross-entropy loss \cite{achille2018emergence,harutyunyan2020improving}.
    which also decomposes as \citep{achille2018emergence,harutyunyan2020improving}
%    The following equation shows the decomposition of the expected cross-entropy loss.
    \begin{align}
    \label{eq:cross-entropy decomposition} 
           \ex{W,Z_i}{\ell(f_W(T_i),Y_i)} = H(Y_i|T_i)+\ex{T_i,W}{\mathrm{D_{KL}}(P_{Y_i|T_i,W}||Q_{Y_i|T_i,W})}-I(W;Y_i|T_i).
    \end{align}
    %Similar results have been given in previous works \cite{achille2018emergence,harutyunyan2020improving}, 
    Then minimizing the expected cross-entropy loss may not adequately reduce $H(Y_i|T_i)$ but rather cause $I(W;Y_i|T_i)$ to significantly increase, particularly when the model capacity is large. This may have two negative effects. First, the condition $I(X;Y)=I(T;Y)$ is significantly violated,
     and $\mathrm{D_{KL}}(P_{Y'|T'}||P_{Y|T})$ is no longer upper bounded by $\mathrm{D_{KL}}(P_{Y'|X'}||P_{Y|X})$. 
    %  As a consequence, 
     Hence, aligning the two marginals alone may not be adequate. Second,
    %and $I(W;Y_i|T_i)$ is proved to highly related to overfitting. 
    %In particular, under the empirical risk minimization setting, the learning algorithm can easily reach the zero cross entropy loss by making $I(W;Y_i|T_i)$ arbitrarily large (until getting close to its upper bound $\log|\mathcal{Y}|$). 
   %Intuitively, 
   large $I(W;Y_i|T_i)$ indicates $W$ just simply memorizes 
  the label 
   $Y_i$, resulting a form of overfitting and hurting the generalization performance. 
   %In this case, one may not obtain a minimal $H(Y|T)$ (and the second term in Eq.~\ref{eq:cross-entropy decomposition}) by minimizing the cross-entropy loss. 
   
   The key take-away from the above analysis is that when aligning the marginals in UDA, controlling the source label information in the weights can be important to achieve good cross-domain generalization. A similar message can also be deduced from Theorem~\ref{thm:ood-semi-bound-1}, when it is viewed in the representation space and noting 
   $I^{T_j'}(W;Z_i)=I^{T_j'}(W;T_i)+I^{T_j'}(W;Y_i|T_i)$.
   
   %we also need to restrict the source label information to improve the generalization performance.

    % So far we do not discuss the role of the source data label, from Lemma \ref{lem:cmi-decomposition}, we can see that controlling label information of source domain will also help to control $I(W;Z_i|X_j')$. Intuitively, $I(W;Y_i|X_i)$ is the information of noisy component of labels contained in the weights. To allow the hypothesis $W$ generalize to target domain, we do not want it to overfit the source training sample's noisy label.

   To control label information, \cite{harutyunyan2020improving} proposed an approach called LIMIT.
   %, refers to limiting label information memorization in training. 
%   Roughly speaking, to update the parameters of the classifier, they construct an auxiliary network to predict the gradient instead of using the real gradient, in which case the true label is not directly used for training the classifier.
   However, this method is rather complicated and arguably hard to train in domain adaptation (see Appendix~\ref{sec:limit}). We now derive a simple alternative strategy for this purpose. 
%   Here we can also invoke an auxiliary network, for the sake of simplification, we use this network to provide the pseudo labels of source domain data, rather than providing the pseudo gradients, to our classifier. Clearly, we still need to make ${R}_S(w)$ small in the current strategy, so we need the pseudo label $\widetilde{Y}$ to be gradually close to the real label $Y$, which means the auxiliary network should be optimized to fit source domain sample simultaneously.
   
%   That is to say, we use $\widetilde{Z}=(X,\widetilde{Y})$ as a training instance, and $\widetilde{Y}$ is the output of the auxiliary network, or by an abuse of the notation, $\widetilde{Y}=\tilde{f}_S(X)$, the subscript $S$ means this auxiliary network is allowed to see the source domain sample. In this case, the dependence between hypothesis $W$ and source label $Y$ is reduced, making $I(W;Y_i|X_i)$ smaller. It seems that any $\widetilde{Y}$ should work for reducing  $I(W;Y_i|X_i)$. However, it's worth reminding that $\mathrm{Err}=\ex{}{R_{\mu'}(w)-\hat{R}_S(w)}$, to let $R_{\mu'}(w)$ small, we need both the upper bound of $\mathrm{Err}$ and $\hat{R}_S(w)$ to be small. Thus, to making $\hat{R}_S(w)$ small in the current strategy, we still need $\widetilde{Y}$ to be close to the real label $Y$, which means the auxiliary network should be optimized to fit source domain sample simultaneously.

Notice that
$
I^{T_j'}(W;Y_i|T_i)
\leq\inf_Q\ex{T_i}{\mathrm{D_{KL}}(P_{W|Y_i,T_i,T_j'=t_j'}||Q_{W|T_i,T_j'=t_j'})},    
$ which is a simple extension of variational representation of mutual information \cite[Corollary~3.1.]{polyanskiy2019lecture}. Here $Q$ could be any distribution. By assuming $P=\mathcal{N}(W,\sigma^2\mathrm{I}_d|Y_i,T_i,T_j'=t_j')$ and taking $Q=\mathcal{N}(\widetilde{W},\tilde{\sigma}^2\mathrm{I}_d|T_i,T_j'=t_j')$, we have
\[
I^{T_j'}(W;Y_i|T_i)\leq\inf_Q\ex{T_i}{\mathrm{D_{KL}}(P_{W|Y_i,T_i,T_j'=t_j'}||Q_{\tilde{W}|T_i,T_j'=t_j'})}\propto ||W-\widetilde{W}||^2.
% \vspace{-1mm}
\]
Thus, 
%to control the label information, one can control the distance between $W$ and $\widetilde{W}$. Specifically, 
we may create an auxiliary classifier $f_{\widetilde{w}}$ that is not allowed to access to the real source label $Y$. In each iteration, we use the pseudo labels of target data (and source data) assigned by $f_{{w}}$ to train $f_{\widetilde{w}}$ and adding $||W-\widetilde{W}||^2$ as a regularizer in the training of $W$. The algorithm is given in the Appendix. Remarkably the regularizer here resembles ``Projection Norm'' designed in \cite{yu2022predicting} for out-of-distribution generalization. 

% To be more precisely, recall that $I^{T'}(W_T;Y|T)\leq I^{T'}(W_{0:T};Y|X)=\sum_{t=1}^T I^{T'}(W_{t};Y|T,W_{t-1})$. At the beginning of the training, $I^{T'}(W_{t};Y|T,W_{t-1})$ is extremely small and the distribution of pseudo label $\widetilde{Y}$ is far from that of $Y$. This will make the overall information $I^{T'}(W_T;Y|T)$ smaller. At the end of training, $\widetilde{Y}$ is nearly the same with $Y$, then ${R}_S(w)$ will be small.

%% file: Experiments.tex
\section{Experimental Results}
\label{sec:experiment}
We perform experiments to verify the  proposed techniques inspired by our theory.
%in the previous section.  The experimental setup follows that in \cite{nguyen2022kl}. 

\begin{table*}[t!]
%  \footnotesize
 \scriptsize
 \centering
 \caption{\small RotatedMNIST and Digits. Results of baselines are reported from \cite{nguyen2022kl}.}
%  \vspace{-0.08in}
  \label{tab:RM-Digits}
 \centering
 	\begin{tabular}{ccccccccccc}
 		\toprule
 		& \multicolumn{6}{c}{RotatedMNIST ($\mathbf{0^{\circ}}$ as source domain)} & \multicolumn{4}{c}{Digits}\\
		\cmidrule(r){2-7}
		\cmidrule(r){8-11}
		Method  & $\mathbf{15^{\circ}}$  & $\mathbf{30^{\circ}}$   & $\mathbf{45^{\circ}}$ & $\mathbf{60^{\circ}}$ & $\mathbf{75^{\circ}}$ & \textbf{Ave} & \textbf{M $\rightarrow$ U}  & \textbf{U $\rightarrow$ M}  & \textbf{S $\rightarrow$ M}  & \textbf{Ave}   \\
		\midrule
		ERM & 97.5±0.2 & 84.1±0.8 & 53.9±0.7 & 34.2±0.4 & 22.3±0.5 & 58.4 & 73.1±4.2 & 54.8±6.2 & 65.9±1.4 & 64.6 \\
		DANN & 97.3±0.4 & 90.6±1.1 & 68.7±4.2 & 30.8±0.6 & 19.0±0.6 & 61.3 & 90.7±0.4 & 91.2±0.8 & 71.1±0.5 & 84.3 \\
		MMD & 97.5±0.1 & 95.3±0.4 & 73.6±2.1 & 44.2±1.8 & 32.1±2.1 & 68.6 & 91.8±0.3 & 94.4±0.5 & 82.8±0.3 & 89.7 \\
		CORAL & 97.1±0.3 & 82.3±0.3 & 56.0±2.4 & 30.8±0.2 & 27.1±1.7 & 58.7 & 88.0±1.9 & 83.3±0.1 & 69.3±0.6 & 80.2 \\
		WD & 96.7±0.3 & 93.1±1.2 & 64.1±3.3 & 41.4±7.6 & 27.6±2.0 & 64.6 & 88.2±0.6 & 60.2±1.8 & 68.4±2.5 & 72.3 \\
		KL  & {97.8±0.1} & {97.1±0.2} & {93.4±0.8} & {75.5±2.4} & {68.1±1.8} & {86.4} & {98.2±0.2}& {97.3±0.5} & {92.5±0.9} & {96.0} \\
		\midrule
		ERM-GP & {97.5±0.1} & {86.2±0.5} & {62.0±1.9} & {34.8±2.1} & {26.1±1.2} & {61.2} & {91.3±1.6} & 72.7±4.2 & {68.4±0.2} &  77.5 \\
		KL-GP  & {98.2±0.2} & {96.9±0.1} & {95.0±0.6} & \textbf{88.0±8.1} & \textbf{78.1±2.5} & \textbf{91.2} & {98.8±0.1}& \textbf{97.8±0.1} & \textbf{93.8±1.1} & \textbf{96.8} \\
		KL-CL  & \textbf{98.4±0.2} & \textbf{97.3±0.2} & \textbf{95.6±0.1} & {83.0±8.2} & 73.6±4.0 & 89.6 & \textbf{98.9±0.1}& 97.7±0.1 & 93.0±0.3 & 96.5 \\
 		\bottomrule
 	\end{tabular}
 	% \vspace{-2mm}
\end{table*}

\vspace{-2mm}
\paragraph{Datasets} We select two popular small datasets, RotatedMNIST and Digits, to compare the different methods. RotatedMNIST is built based on the MNIST dataset \citep{lecun2010mnist} and consists of six domains, each containing $11,666$ images. These six domains are rotated MNIST images with rotation angle $0^{\circ},15^{\circ},30^{\circ},45^{\circ},60^{\circ}$ and $75^{\circ}$, respectively. We will take the original MNIST dataset ($0^{\circ}$) as the source domain and take other five domains as target domains. Hence, there are five domain adaptation tasks on RotatedMNIST. Digits consists of three sub-datasets, namely MNIST, USPS \citep{hull1994database} and SVHN \citep{netzer2011reading}, and the corresponding domain adaptation tasks are MNIST$\rightarrow$USPS (\textbf{M$\rightarrow$U}), USPS$\rightarrow$MNIST (\textbf{U$\rightarrow$M}), SVHN$\rightarrow$MNIST (\textbf{S$\rightarrow$M}).
% \vspace{-1mm}
% \textcolor{red}{We also consider a large dataset Office-Home \cite{venkateswara2017deep}.}

\vspace{-2mm}
\paragraph{Compared Methods} Baseline methods are some popular marginal alignment UDA methods including \textbf{DANN} \citep{ganin2016domain}, \textbf{MMD} \citep{li2018domain}, \textbf{CORAL} \citep{sun2016deep}, \textbf{WD} \citep{shen2018wasserstein} and \textbf{KL} \citep{nguyen2022kl}. We also choose \textbf{ERM} as another baseline, in which only the source-domain sample is accessible during training. To verify the strategies inspired by our theory, we first add the gradient penalty to the ERM algorithm (\textbf{ERM-GP}), and we then combine gradient penalty (GP) and controlling label information (CL) with the recent proposed KL guided marginal alignment method, which are denoted by \textbf{KL-GP} and \textbf{KL-CL}, respectively.
% \vspace{-1mm}

\vspace{-2mm}
\paragraph{Implementation Details} Most of our implementation is based on the \textit{DomainBed} suite \citep{gulrajani2021in}. Other settings  exactly follow \cite{nguyen2022kl} and the results of baseline methods are taken from \cite{nguyen2022kl}. Specifically, each algorithm is run three times and we show the average performance with the error bar. Every dataset has a validation set, and the model selection scheme is based on the best performance achieved on the validation set of target domain during training (oracle). The hype-parameter searching process is also built upon the implementation in the \textit{DomainBed} suite. Other details and additional experiments can be found in Appendix.
% \vspace{-1mm}

\vspace{-2mm}
\paragraph{Results} From Table~\ref{tab:RM-Digits}, we first notice that gradient penalty allows \textbf{ERM} to perform more comparably to other marginal alignment methods. For example, on RotatedMNIST, \textbf{ERM-GP} outperforms \textbf{CORAL} and performs nearly the same with \textbf{DANN}. On Digits, \textbf{ERM-GP} outperforms \textbf{WD}. When GP and CL combined with KL guided algorithm, we can see that the performance can be further boosted. This justifies the discussion in Section~\ref{sec:application-1} and Section~\ref{sec:application-2}.

%% file: Appendix.tex
\appendix
\addcontentsline{toc}{section}{Appendix} % Add the appendix text to the document TOC
\part{Appendix} % Start the appendix part
\parttoc

\section{Some Prerequisite Definitions and Useful Lemmas}
\begin{defn}[Wasserstein Distance]
Let $d(\cdot,\cdot)$ be a metric and let $P$ and $Q$ be probability measures on $\mathcal{X}$. Denote  $\Gamma(P,Q)$ as the set of all couplings of $P$ and $Q$ (i.e. the set of all joint distributions  on $\mathcal {X} \times \mathcal {X}$ with two marginals being $P$ and $Q$), then the Wasserstein Distance of order one between $P$ and $Q$ is defined as $\mathbb{W}(P,Q)\triangleq\inf_{\gamma\in\Gamma(P,Q)}\int_{\mathcal{X}\times\mathcal{X}} d(x,x')d\gamma(x,x')$. 
\end{defn}
\begin{rem}
Similar to \cite{rodriguez2021tighter}, here we mainly focus on $1$-Wasserstein distance but all the upper bounds based on $1$-Wasserstein distance also holds for higher order Wasserstein distance by H\"older's inequality \cite[Remark~6.6]{Villani2008}.
\end{rem}

\begin{defn}[Total Variation]
The total variation between two probability measures  $P$ and $Q$ is $\mathrm{TV}(P,Q)\triangleq\sup_E\left|P(E)-Q(E)\right|$, where the supremum is over all measurable set $E$.
\end{defn}

\begin{rem}
    Note that the total variation equals to the Wasserstein distance under the discrete metric (or Hamming distortion)  $d(x,x')=\mathbbm{1}_{x\neq x'}$ where $\mathbbm{1}$ is the indicator function \cite[Theorem~6.15]{Villani2008}.
\end{rem}

% \begin{defn}[Lautum Information {\citep{palomar2008lautum}}]
%  Define the lautum information between $X$ and $Y$ as $L(X;Y)\triangleq\mathrm{D_{KL}}(P_XP_Y||P_{XY})$.
% \end{defn}

The key quantity in the most information-theoretic generalization bounds is the mutual information between algorithm's input and output.
%or equivalently, the expected KL divergence between some posterior distribution and prior distribution. 
Specifically, the core technique behind these bounds is the well-known Donsker-Varadhan representation of KL divergence \cite[Theorem~3.5]{polyanskiy2019lecture}.
\begin{lem}[Donsker and Varadhan's variational formula]
\label{lem:DV-KL}
Let $Q$, $P$ be probability measures on $\Theta$, for any bounded measurable function $f:\Theta\rightarrow \mathbb{R}$, we have
$
\mathrm{D_{KL}}(Q||P) = \sup_{f} \ex{\theta\sim Q}{f(\theta)}-\log\ex{\theta\sim P}{\exp{f(\theta)}}.
$
\end{lem}
\begin{rem}
Motivated by the classic $f$-divergence, \citet{acuna2021f} proposed a discrepancy measure called $\mathrm{D}_{\mathcal{H}}^\phi$-discrepancy (or $\mathrm{D}_{h,\mathcal{H}}^\phi$-discrepancy). 
As KL divergence belongs to the family of $f$-divergences and both \cite{acuna2021f} and our work use the variational representation of divergence, there appears to be a connection between our work (in Section~\ref{sec:unexpected-bound}) and theirs. However, it's important to note that the variational characterization of $f$-divergence used in \cite{acuna2021f} is based on the results of \cite{nguyen2010estimating}, while the Donsker-Varadhan representation of KL divergence (see Lemma~\ref{lem:DV-KL}) used in our paper cannot be directly obtained from their variational characterization \citep{jiao2017dependence,agrawal2020optimal}. In fact, simply choosing $x\log x$ as the conjugate function would result in a weaker bound than Lemma~\ref{lem:DV-KL}. Therefore, while there is some similarity between our results and those of \cite{acuna2021f}, our results in Section~\ref{sec:unexpected-bound} cannot be directly derived from theirs.
% Since KL divergence belongs to the family of $f$-divergences (e.g., choosing $x\log x$ as the Fenchel conjugate function) and both \cite{acuna2021f} and our work invoke the variational representation of the divergence, it seems that our work (in Section~\ref{sec:unexpected-bound}) is related to theirs. However, the variational characterization of $f$-divergence used in \cite{acuna2021f} is based on the result of \cite{nguyen2010estimating}, and the Donsker-Varadhan representation of KL divergence (see Lemma~\ref{lem:DV-KL}) used in this paper cannot be directly recovered from their variational characterization \citep{jiao2017dependence,agrawal2020optimal}. In fact, simply choosing $x\log x$ as the conjugate function will lead to a weaker bound than Lemma~\ref{lem:DV-KL}. Thus, our results (in Section~\ref{sec:unexpected-bound}) cannot be directly recovered from the results in \cite{acuna2021f}.
\end{rem}

Similar to \citet[Lemma~1.]{xu2017information}, we need the following lemma as a main tool.
\begin{lem}
\label{lem:DV-Subgaussian}
Let $Q$ and $P$ be probability measures on $\Theta$. Let $\theta'\sim Q$ and $\theta\sim P$. If $g(\theta)$ is $R$-subgaussian, then,
\[
\left|\ex{\theta'\sim Q}{g(\theta')}-\ex{\theta\sim P}{g(\theta)}\right|\leq \sqrt{2R^2\mathrm{D_{KL}}(Q||P)}.
\]
\end{lem}
\begin{proof}
Let $f=t\cdot g$ for any $t\in \mathbb{R}$, by Lemma \ref{lem:DV-KL}, we have
\begin{align*}
        \mathrm{D_{KL}}(Q||P) \geq& \sup_t\ex{\theta'\sim Q}{t g(\theta')}-\log{\ex{\theta\sim P}{\exp{t\cdot g(\theta)}}}\\
        =&\sup_t\ex{\theta'\sim Q}{t g(\theta')}-\log{\ex{\theta\sim P}{\exp{t (g(\theta)-\ex{\theta\sim P}{g(\theta)}+\ex{\theta\sim P}{g(\theta)})}}}\\
        =&\sup_t \ex{\theta'\sim Q}{t g(\theta')}-\ex{\theta\sim P}{t g(\theta)}-\log{\ex{\theta\sim P}{\exp{t (g(\theta)-\ex{\theta\sim P}{g(\theta)})}}}\\
        \geq&\sup_t t\left(\ex{\theta'\sim Q}{ g(\theta')}-\ex{\theta\sim P}{ g(\theta)}\right)-t^2R^2/2,
\end{align*}
where the last inequality is by the subgaussianity of $g(\theta)$.

Then consider the case of $t>0$ and $t<0$ ($t=0$ is trivial), by AM–GM inequality (i.e. the arithmetic mean is greater than or equal to the geometric mean), the following is straightforward,
\[
\left|\ex{\theta'\sim Q}{g(\theta')}-\ex{\theta\sim P}{g(\theta)}\right|\leq \sqrt{2R^2\mathrm{D_{KL}}(Q||P)}.
\]
This completes the proof.
\end{proof}

The following lemma is the Kantorovich–Rubinstein duality of Wasserstein distance \citep{Villani2008}.
\begin{lem}[KR duality]
\label{lem:KR duality}
For any two distributions $P$ and $Q$, we have
\[
\mathbb{W}(P,Q)=\sup_{f\in\mathrm{1-Lip(\rho)}}\int_\mathcal{X} f dP - \int_\mathcal{X} f dQ,
\]
where the supremum is taken over all $1$-Lipschitz functions in the metric $d$, i.e. $|f(x)-f(x')|\leq d(x,x')$ for any $x,x'\in\mathcal{X}$.
\end{lem}

To connect total variation with KL divergence , we will use Pinsker's inequality \cite[Theorem~6.5]{polyanskiy2019lecture} and Bretagnolle-Huber inequality \cite[Lemma~2.1]{bretagnolle1979estimation} in this paper, for more discussion about these two inequalities, we refer readers to \cite{canonne2022short}.

\begin{lem}[Pinsker's inequality]
\label{lem:pinsker}
$\mathrm{TV}(P,Q)\leq\sqrt{\frac{1}{2}\mathrm{D_{KL}}(P||Q)}$.
\end{lem}

\begin{lem}[Bretagnolle-Huber inequality]
\label{lem:BH}
$\mathrm{TV}(P,Q)\leq\sqrt{1-e^{-\mathrm{D_{KL}}(P||Q)}}$.
\end{lem}

Below is the variational formula (or golden formula) of mutual information.
\begin{lem}[{\citet[Corollary~3.1.]{polyanskiy2019lecture}}]
\label{lem:mi-center-gravity}
For two random variables $X$ and $Y$, we have
\[
I(X;Y) = \inf_{P} \ex{X}{\mathrm{D_{KL}}(Q_{Y|X}||P)},
\]
where the infimum is achieved at $P=Q_Y$.
\end{lem}

\begin{figure}
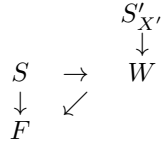

\begin{center}
{
    \begin{tabular}{clc}
          &  & $S'_{X'}$  \\
          & &$\downarrow$ \\
         $S$ &  $\rightarrow$ & $W$  \\
           $\downarrow$& $\swarrow$&  \\
         $F$&  & \\
    \end{tabular}
}
\end{center}
\caption{The relationship between random variables in UDA, where $F=R_{\mu'}(W)-R_S(W)$.}
    \label{fig:Markov}
\end{figure}

\section{Omitted Proofs and Additional Results in Section~\ref{sec:unexpected-bound}}

\subsection{Proof of Theorem~\ref{thm-bound-subgaussian}}
\begin{proof}
    Let $Q=\mu'$, $P=\mu$ and $g=\ell$, then Theorem \ref{thm-bound-subgaussian} comes directly from Lemma~\ref{lem:DV-Subgaussian}.
\end{proof}

\subsection{Proof of Corollary~\ref{cor-bound-symmetric-kl}}
\begin{proof}
    As discussed in Remark~\ref{rem:boundedness}, when the loss is bounded in $[0,M]$, it is guaranteed to be $\frac{M}{2}$-subgaussian for any $w\in\mathcal{W}$. Then, 
    similar to the proof of Theorem~\ref{thm-bound-subgaussian}, let $Q=\mu$, $P=\mu'$ and $g=\ell$ and $R=\frac{M}{2}$, then the following bound holds by Lemma~\ref{lem:DV-Subgaussian},
    \[
    \left|\widetilde{\mathrm{Err}}(w)\right| \leq \sqrt{\frac{M^2}{2}\mathrm{D_{KL}(\mu||\mu')}}.
    \]
    Then, by $\min\{A,B\}\leq\frac{1}{2}(A+B)$, the remaining part is straightforward,
    \[
     \left|\widetilde{\mathrm{Err}}(w)\right| \leq \frac{M}{\sqrt{2}}\sqrt{\min\{\mathrm{D_{KL}}(\mu||\mu'),\mathrm{D_{KL}}(\mu'||\mu)\}} \leq \frac{M}{2}\sqrt{{\mathrm{D_{KL}}(\mu||\mu')+\mathrm{D_{KL}}(\mu'||\mu)}}.
    \]
    This completes the proof.
\end{proof}

\subsection{Proof of Theorem \ref{cor-bound-subgaussian-triangle}}
\begin{proof}
        Let $w^*=\arg\min_{w\in\mathcal{W}}\ex{Z'}{\ell(f_w(X'),Y')}+\ex{Z}{\ell(f_w(X),Y)}$. By Lemma \ref{lem:DV-KL},
        \begin{align*}
             \mathrm{D_{KL}}(P_{X'}||P_{X})&\geq\sup_{t\in\mathbb{R},w\in\mathcal{W}}\ex{X'}{t\ell(f_w(X'),f_{w^*}(X'))}-\log{\ex{X}{e^{t\ell(f_w(X),f_{w^*}(X))}}}.
        \end{align*}
       Recall that $\ell(f_{w'}(X),f_w(X))$ is $R$-subgaussian, by using Lemma~\ref{lem:DV-Subgaussian} (let $Q=P_{X'}$, $P=P_{X}$ and $g(\cdot)=\ell(f_{w'}(\cdot),f_{w}(\cdot))$), we have
       \begin{align}
           \left|\ex{X'}{\ell(f_{w}(X'),f_{w^*}(X'))}-\ex{X}{\ell(f_{w}(X),f_{w^*}(X))}\right|\leq\sqrt{2R^2\mathrm{D_{KL}}(P_{X'}||P_{X})}.
           \label{ineq:dis-optimal}
       \end{align}
       
       For any $f_w\in\mathcal{F}$, by the symmetric and triangle property of the loss, we have
       \begin{align}
           &\ex{Z'}{\ell(f_{w}(X'),Y')}\notag\\
           \leq&\ex{X'}{\ell(f_{w}(X'),f_{w^*}(X'))}+\ex{Z'}{\ell(f_{w^*}(X'),Y')}\notag\\
           \leq&\ex{X}{\ell(f_{w}(X),f_{w^*}(X))}+\sqrt{2R^2\mathrm{D_{KL}}(P_{X'}||P_{X})}+\ex{Z'}{\ell(f_{w^*}(X'),Y')}\label{ineq:dis-optimal-2}\\
           =&\int_x \ell(f_{w}(x),f_{w^*}(x)) dP_X(x)+\sqrt{2R^2\mathrm{D_{KL}}(P_{X'}||P_{X})}+\ex{Z'}{\ell(f_{w^*}(X'),Y')}\notag\\
           =&\int_x\int_y \ell(f_{w}(x),f_{w^*}(x))dP_{Y|X=x}(y) dP_X(x)+\sqrt{2R^2\mathrm{D_{KL}}(P_{X'}||P_{X})}+\ex{Z'}{\ell(f_{w^*}(X'),Y')}\notag\\
           \leq&\int_x\int_y \ell(f_{w}(x),y)+\ell(y,f_{w^*}(x)) dP_{Y|X=x}(y) dP_X(x)+\sqrt{2R^2\mathrm{D_{KL}}(P_{X'}||P_{X})}+\ex{Z'}{\ell(f_{w^*}(X'),Y')}\label{ineq:triangle-1}\\
           =&\ex{Z}{\ell(f_{w}(X),Y)}+\ex{Z}{\ell(Y,f_{w^*}(X))}+\sqrt{2R^2\mathrm{D_{KL}}(P_{X'}||P_{X})}+\ex{Z'}{\ell(f_{w^*}(X'),Y')},\notag
       \end{align}
       where Eq. (\ref{ineq:dis-optimal-2}) is by Eq. (\ref{ineq:dis-optimal}) and Eq. (\ref{ineq:triangle-1}) is again by the triangle property of the loss function.
       
       Thus, $\widetilde{\mathrm{Err}}(w) \leq \sqrt{2R^2\mathrm{D_{KL}(P_{X'}||P_{X})}}+\lambda^*$, which completes the proof.
        \end{proof}
        
\subsection{Proof of Theorem~\ref{cor:kl-expected-disagreement}}
\begin{proof} By Lemma~\ref{lem:DV-KL},
    \begin{align*}
             \mathrm{D_{KL}}(P_{X'}||P_{X})&\geq\sup_{t\in\mathbb{R},w,w'\in\mathcal{W}^2}\ex{X'}{t\ell(f_{w}(X'),f_{w'}(X'))}-\log{\ex{X}{e^{t\ell(f_{w}(X),f_{w'}(X))}}}\\
             &\geq\sup_{t\in\mathbb{R}}\ex{W,W'}{\ex{X'}{t\ell(f_{W}(X'),f_{W'}(X'))}-\log{\ex{X}{e^{t\ell(f_{W}(X),f_{W'}(X))}}}}\\
             &\geq\sup_{t\in\mathbb{R}}{\ex{W,W'}{\ex{X'}{t\ell(f_{W}(X'),f_{W'}(X'))}}-\log{\ex{W,W'}{\ex{X}{e^{t\ell(f_{W}(X),f_{W'}(X))}}}}},
        \end{align*}
        where the last inequality is by applying Jensen's inequality to the logarithm function, which is concave.
        
      By the subgaussian assumption,
    %   \textcolor{red}{need more careful development}
       \[
       \left|\ex{W,W'}{\ex{X'}{\ell(f_W(X'),f_{W'}(X'))}}-\ex{W,W'}{\ex{X}{\ell(f_W(X),f_{W'}(X))}}\right|\leq\sqrt{2R^2\mathrm{D_{KL}}(P_{X'}||P_{X})}.
       \]
       This concludes the proof.
    \end{proof}

\subsection{Proof of Theorem~\ref{thm-bound-lipschitz}}
\begin{proof}
From the definition, we have
    \begin{align*}
        \left|\widetilde{\mathrm{Err}}(w)\right|=&\left|\ex{Z'}{\ell(f_w(X'),Y')}-\ex{Z}{\ell(f_w(X),Y)}\right|\\
        \leq& \beta\mathbb{W}(\mu,\mu').
    \end{align*}
where the last inequality is by the KR duality of Wasserstein distance (see Lemma~\ref{lem:KR duality}).
    \end{proof}

\subsection{Proof of Corollary~\ref{cor-bound-total-variation-bounded}}
\begin{proof}
    When $d$ is the discrete metric, Wasserstein distance is equal to the total variation, then by Theorem~\ref{thm-bound-lipschitz},
    \[
    \left|\widetilde{\mathrm{Err}}(w)\right| \leq \beta\mathrm{TV}(\mu',\mu),
    \]
    The remaining part is by using Lemma~\ref{lem:pinsker} and Lemma~\ref{lem:BH}:
    \[
    \beta\mathrm{TV}(\mu',\mu)\leq \beta\sqrt{\min\left\{\frac{1}{2}\mathrm{D_{KL}}(\mu'||\mu),1-e^{-\mathrm{D_{KL}}(\mu'||\mu)}\right\}}.
    \]
    
    Then, if $\ell$ is bounded by $M$, we can replace $\beta$ by $M$ above, which completes the proof.
\end{proof}
    
\subsection{Proof of Theorem~\ref{thm-bound-lipshitz-triangle}}
\begin{proof}
    Let $w^*=\arg\min_{w\in\mathcal{W}}\ex{Z'}{\ell(f_w(X'),Y')}+\ex{Z}{\ell(f_w(X),Y)}$. 
    
    If $\ell(f_w(X),f_{w'}(X))$ is $\beta$-Lipschitz in $\mathcal{X}$ for any $w,w'\in \mathcal{W}$, then similar to Theorem~\ref{thm-bound-lipschitz}, it's easy to show that
    \begin{align}
    \label{ineq:dis-optimal-3}
        \ex{X'}{\ell(f_w(X'),f^*(X'))}-\ex{X}{\ell(f_w(X),f^*(X))}\leq \beta \mathbb{W}(P_X',P_X)
    \end{align}
       
       For any $f_w\in\mathcal{F}$, by the symmetric and triangle property of the loss, we have
       \begin{align}
           &\ex{Z'}{\ell(f_{w}(X'),Y')}\notag\\
           \leq&\ex{X'}{\ell(f_{w}(X'),f_{w^*}(X'))}+\ex{Z'}{\ell(f_{w^*}(X'),Y')}\notag\\
           \leq&\ex{X}{\ell(f_{w}(X),f_{w^*}(X))}+\beta \mathbb{W}(P_X',P_X)+\ex{Z'}{\ell(f_{w^*}(X'),Y')}\label{ineq:dis-optimal-4}\\
           \leq&\ex{Z}{\ell(f_{w}(X),Y)}+\ex{Z}{\ell(Y,f_{w^*}(X))}+\beta \mathbb{W}(P_X',P_X)+\ex{Z'}{\ell(f_{w^*}(X'),Y')},\notag
       \end{align}
       where Eq. (\ref{ineq:dis-optimal-4}) is by Eq. (\ref{ineq:dis-optimal-3}) and the last inequality is again by the triangle property of the loss function.
       This completes the proof.
\end{proof}

\subsection{Additional Results: Sample Complexity Bounds}
\label{sec:sample-complexity}
One of the main ingredients to derive our sample complexity bound is the following lemma, where a concentration bound for a class of unbounded functions is given.
\begin{lem}[{\citet[Corollary~9]{cortes2019relative}}]
\label{lem:pseudo-unbounded-bound}
Let $\kappa>2$ and $\mathcal{G}=\{g:\mathcal{Z}\rightarrow\mathbb{R} \text{ s.t. } \ex{\mu}{e^{g(Z)}}< +\infty\}$. Assume $\ex{\mu}{g(Z)^\kappa}< +\infty$ for all $g\in \mathcal{G}$. Let $\hat{\mu}$ be the empirical distributions consist of $n$ data points sampled i.i.d. from $\mu$. If $\mathcal{G}$ has the finite pseudo-dimension $d$, then for $\forall \delta\in(0,1)$, the following
inequality holds for all $g\in \mathcal{G}$ with probability at least $1-\delta$,
\begin{align*}
    \ex{\mu}{g(Z)} \leq  \ex{\hat{\mu}}{g(Z)} + 2\Lambda(\kappa)\sqrt[\kappa]{\ex{\mu}{g(Z)^\kappa}}\sqrt{\frac{1}{n}\left(d\log{\frac{2en}{d}}+\log{\frac{4}{\delta}}\right)},
\end{align*}
where $\Lambda(\kappa)=\left(\frac{1}{2}\right)^\frac{2}{\kappa}\left(\frac{\kappa}{\kappa-2}\right)^\frac{\kappa-1}{\kappa}$.
\end{lem}

Below is another useful lemma for the bounded case, which comes from \citet[Theorem~11.8]{mohri2018foundations}  with a slight modification (by invoking a different VC-dimension based generalization bound from \cite{SLT98Vapnik}).
\begin{lem}
\label{lem:pseudo-bounded-bound}
Let $\mathcal{F}=\{f:\mathcal{Z}\rightarrow\mathbb{R}^+\}$. Assume $\ex{\mu}{f(Z)}< M$ for all $f\in\mathcal{F}$ for some constant $M>0$. Let $\hat{\mu}$ be the empirical distributions consist of $n$ data points sampled i.i.d. from $\mu$. If $\mathcal{F}$ has the finite pseudo-dimension $d$, then for $\forall \delta\in(0,1)$, the following
inequality holds for all $f\in \mathcal{F}$ with probability at least $1-\delta$,
\begin{align*}
    \ex{\mu}{f(Z)} \leq  \ex{\hat{\mu}}{f(Z)} + 2M\sqrt{\frac{1}{n}\left(d\log{\frac{2en}{d}}+\log{\frac{4}{\delta}}\right)}.
\end{align*}
\end{lem}

 We are now in a position to state our sample complexity bound.
\begin{thm}
\label{thm:empirical-kl}
Let $\hat{\mu}$ and $\hat{\mu}'$ be the empirical distributions consist of $n$ source data and $m$ target data sampled i.i.d. from $\mu$ and $\mu'$, respectively. Let $\mathcal{G}=\{g:\mathcal{Z}\rightarrow\mathbb{R} \text{ s.t. } \ex{\mu}{e^{g(Z)}}< \infty\}$ with finite pseudo-dimension $d_1$, and let the pseudo-dimension of $\{\exp{}\circ g | g\in\mathcal{G}\}$ be $d_2$. Let $\kappa>2$ and assume that $\ex{\mu}{g(Z)^\kappa}< +\infty$ for all $g\in \mathcal{G}$. Assume there exists a constant $\alpha\le \min_{g\in\mathcal{G}}\{\ex{\hat{\mu}}{e^{g(Z)}},\ex{\mu}{e^{g(Z)}}\}$. Then for $\forall \delta\in(0,1)$ the following bound holds with probability at least $1-\delta$,
\[
\mathrm{D_{KL}}(\mu'||\mu)-\mathrm{D_{KL}}(\hat{\mu}'||\hat{\mu})\leq C_1(\kappa)\sqrt{\frac{1}{n}\left(d_1\log{\frac{2en}{d_1}}+\log{\frac{4}{\delta}}\right)}+C_2(\alpha) \sqrt{\frac{1}{m}\left(d_2\log{\frac{2em}{d_2}}+\log{\frac{4}{\delta}}\right)},
\]
where $C_1(\kappa)=\left(\frac{1}{2}\right)^\frac{2-\kappa}{\kappa}\left(\frac{\kappa}{\kappa-2}\right)^\frac{\kappa-1}{\kappa}\sup_{g\in\mathcal{G}}\sqrt[\kappa]{\ex{\mu}{g(Z)^\kappa}}$ and $C_2(\alpha)=\frac{2}{\alpha}\sup_{g\in\mathcal{G}}\ex{\mu}{e^{g(Z)}}$.
\end{thm}
\begin{proof}
    Recall Lemma~\ref{lem:DV-KL}, we have
    \[
    \mathrm{D_{KL}}(\mu'||\mu)=\sup_{g\in\mathcal{G}} \ex{\mu'}{g(Z')} - \log{\ex{\mu}{e^{g(Z)}}},
    \]
    and
    \[
    \mathrm{D_{KL}}(\hat{\mu}'||\hat{\mu})=\sup_{g\in\mathcal{G}} \ex{\hat{\mu}'}{g(Z')} - \log{\ex{\hat{\mu}}{e^{g(Z)}}}.
    \]
    
    Then, with the probability at least $1-\delta$,
    \begin{align}
        &\mathrm{D_{KL}}(\mu'||\mu)-\mathrm{D_{KL}}(\hat{\mu}'||\hat{\mu})\notag\\
        =&\sup_{g\in\mathcal{G}} \ex{\mu'}{g(Z')} - \log{\ex{{\mu}}{e^{g(Z)}}} - \left(\sup_{g\in\mathcal{G}} \ex{\hat{\mu}'}{g(Z')} - \log{\ex{\hat{\mu}}{e^{g(Z)}}}\right)\notag\\
        \leq& \sup_{g\in\mathcal{G}} \ex{\mu'}{g(Z')} - \log{\ex{\mu}{e^{g(Z)}}} - \left(\ex{\hat{\mu}'}{g(Z')} - \log{\ex{\hat{\mu}}{e^{g(Z)}}}\right)\notag\\
        =& \sup_{g\in\mathcal{G}} \ex{\mu'}{g(Z')}- \ex{\hat{\mu}'}{g(Z')}+ \log{\ex{\hat{\mu}}{e^{g(Z)}}} - \log{\ex{\mu}{e^{g(Z)}}}  \notag\\
        \leq& \sup_{g\in\mathcal{G}} \left|\ex{\mu'}{g(Z')}- \ex{\hat{\mu}'}{g(Z')}\right|+ \sup_{g\in\mathcal{G}}\left|\log{\ex{\hat{\mu}}{e^{g(Z)}}} - \log{\ex{\mu}{e^{g(Z)}}}\right|  \notag\\
        \leq& \sup_{g\in\mathcal{G}} \left|\ex{\mu'}{g(Z')}- \ex{\hat{\mu}'}{g(Z')}\right|+ \sup_{g\in\mathcal{G}}\frac{1}{\alpha}\left|\ex{\hat{\mu}}{e^{g(Z)}} - \ex{\mu}{e^{g(Z)}}\right|\label{ineq:lipschiz-kl}\\
        \leq& C_1(\kappa)\sqrt{\frac{1}{n}\left(d_1\log{\frac{2en}{d_1}}+\log{\frac{4}{\delta}}\right)}+C_2(\alpha) \sqrt{\frac{1}{m}\left(d_2\log{\frac{2em}{d_2}}+\log{\frac{4}{\delta}}\right)},\label{ineq:vc-dim-bound}
    \end{align}
    where Eq. (\ref{ineq:lipschiz-kl}) is derived below. 
    
    W.L.O.G. assume that $\ex{\hat{\mu}}{e^{g(Z)}}\leq\ex{\mu}{e^{g(Z)}}$ (and Eq.~(\ref{ineq:lipschiz-kl}) still holds when $\ex{\hat{\mu}}{e^{g(Z)}}\geq\ex{\mu}{e^{g(Z)}}$), then
    \begin{align*}
        \left|\log \ex{\hat{\mu}}{e^{g(Z)}} - \log \ex{\mu}{e^{g(Z)}}\right| = \left|\log\frac{\ex{\mu}{e^{g(Z)}}}{\ex{\hat{\mu}}{e^{g(Z)}}}\right| =& \left|\log\left(1+\frac{\ex{\mu}{e^{g(Z)}}}{\ex{\hat{\mu}}{e^{g(Z)}}}-1\right)\right|\\
\leq& \left|\frac{\ex{\mu}{e^{g(Z)}}}{\ex{\hat{\mu}}{e^{g(Z)}}}-1\right| \\
=& \left|\frac{1}{\ex{\hat{\mu}}{e^{g(Z)}}}\left(\ex{\mu}{e^{g(Z)}}-\ex{\hat{\mu}}{e^{g(Z)}}\right)\right|\\ 
% \leq \frac{1}{\alpha}\left(\ex{\mu}{e^{g(Z)}}-\ex{\hat{\mu}}{e^{g(Z)}}\right)\right|\\
 \leq& \frac{1}{\alpha}\left|\ex{\mu}{e^{g(Z)}}-\ex{\hat{\mu}}{e^{g(Z)}}\right|.
    \end{align*}
 Eq.~(\ref{ineq:vc-dim-bound}) is by Lemma~\ref{lem:pseudo-unbounded-bound} and Lemma~\ref{lem:pseudo-bounded-bound}. This concludes the proof.
\end{proof}

With Theorem~\ref{thm:empirical-kl} and Theorem~\ref{thm-bound-subgaussian}, we immediately have the following corollary.
\begin{cor}
\label{cor:kl-vc-bound}
Let the conditions in Theorem~\ref{thm:empirical-kl} and Theorem~\ref{thm-bound-subgaussian} hold, then for any $w\in\mathcal{W}$,
             \[
    \left|\widetilde{\mathrm{Err}}(w)\right| \leq \sqrt{2}R\sqrt{\mathrm{D_{KL}}(\hat{\mu}'||\hat{\mu})+C_1(\kappa)\sqrt{\frac{1}{n}\left(d_1\log{\frac{2en}{d_1}}+\log{\frac{4}{\delta}}\right)}+C_2(\alpha) \sqrt{\frac{1}{m}\left(d_2\log{\frac{2em}{d_2}}+\log{\frac{4}{\delta}}\right)}}
    ,\]
    where $C_1(\kappa)$ and $C_2(\alpha)$ are the same as in Theorem~\ref{thm:empirical-kl}.
\end{cor}

\subsection{Additional Discussions on the Convergence of Empirical KL Divergence}
Characterizing the convergence of the empirical KL divergence to the real KL is a challenging task that often requires several additional assumptions, as demonstrated in Theorem~\ref{thm:empirical-kl}. However, it is worth noting that the convergence rate of the empirical distribution to the real distribution in the KL sense is already established in the discrete space. This fact is supported by a classic result in \cite[Theorem~11.2.1]{thomas2006elements}, which we state in the following theorem:
\begin{thm}
\label{thm:convergence-rate-kl}
Let $\hat{\mu}$ and $\hat{\mu}'$ be defined as in Theorem~\ref{thm:empirical-kl}. Assume the space of $\mathcal{Z}$ is finite (i.e. $|\mathcal{Z}|\leq\infty$), then for $\forall \delta\in(0,1)$, with  probability at least $1-\delta$,
\[
\mathrm{D_{KL}}(\hat{\mu}||\mu)\leq \frac{|\mathcal{Z}|}{n}\log{(n+1)}+\frac{1}{n}\log{\frac{1}{\delta}},\qquad  \mathrm{D_{KL}}(\hat{\mu}'||\mu')\leq \frac{|\mathcal{Z}|}{m}\log{(m+1)}+\frac{1}{m}\log{\frac{1}{\delta}}.
\]
\end{thm}
Thus, it suffices to ensure that the empirical KL converge to the real KL with the similar rate, although we do not know if there might exist a faster convergence rate.

\subsection{Generalize to Approximate Triangle Inequality}
\label{sec:generalize-triangle}
In Section~\ref{sec:unexpected-bound}, some results require that the loss obeys the triangle inequality (i.e. Assumption~\ref{ass:triangle}), such as Theorem~\ref{cor-bound-subgaussian-triangle} and Theorem~\ref{thm-bound-lipshitz-triangle}. While the $0-1$ loss satisfies Assumption~\ref{ass:triangle}, some other loss may not. Thus, to generalize Theorem~\ref{cor-bound-subgaussian-triangle} and Theorem~\ref{thm-bound-lipshitz-triangle}, we invoke an approximate triangle inequality, which is originally defined in \cite{crammer2008learning}.
\begin{assum}[$\alpha$-Triangle]
\label{ass:alpha-triangle}
$\ell(\cdot,\cdot)$ is symmetric and satisfies the following $\alpha$-triangle inequality:
    $
    \ell(y_1,y_2)\leq\alpha\left(\ell(y_1,y_3)+\ell(y_3,y_2)\right)~\textit{for any}~y_1,y_2,y_3\in\mathcal{Y}
    $,
    where $\alpha\geq 1$ is a constant that may depend on the hypothesis space $\mathcal{W}$ and the loss $\ell$.
\end{assum}
\begin{rem}
We note that the squared loss satisfies $2$-triangle inequality.
\end{rem}
 Thus, Theorem~\ref{cor-bound-subgaussian-triangle} can be easily generalized below.
\begin{thm}
            \label{thm-bound-subgaussian-alpha-triangle}
            If Assumption \ref{ass:alpha-triangle} holds and let $\ell(f_{w'}(X),f_{w}(X))$ be $R$-subgaussian for any $w,w'\in \mathcal{W}$. Then for any $w$,
             \[
    \widetilde{\mathrm{Err}}(w) \leq (\alpha^2-1)R_\mu+ \alpha\sqrt{2R^2\mathrm{D_{KL}}(P_{X'}||P_{X})}+\alpha^2\lambda^*,
    \] 
    where $\lambda^*=\min_{w\in\mathcal{W}}R_{\mu'}(w)+R_\mu(w)$.
        \end{thm}
Theorem~\ref{thm-bound-lipshitz-triangle} can be generalized in the similar way. 
While Theorem~\ref{thm-bound-subgaussian-alpha-triangle} strictly speaking is not a generalization bound, as it includes $R_\mu$ in the bound, it shares the same underlying concept as Theorem~\ref{cor-bound-subgaussian-triangle}. Namely, to minimize the population risk in the target domain, it is essential for the source domain and target domain to be similar, and for both $R_\mu$ and $\lambda^*$ to be kept small.
% Although strictly speaking, Theorem~\ref{thm-bound-subgaussian-alpha-triangle} is not a generalization bound since $R_\mu$ appears in the bound, it has the same spirit with Theorem~\ref{cor-bound-subgaussian-triangle}, that is, to make the target domain population risk small, we hope source domain and target domain are close to each other, and both $R_\mu$ and $\lambda^*$ are small.
    
\section{Omitted Proofs and Additional Discussions in Section~\ref{sec:expected-bound}}

\subsection{Additional Discussion on Theorem~\ref{thm:ood-semi-bound-1}}

 To derive the bound in Theorem~\ref{thm:ood-semi-bound-1}, we need to make use of the second equality in Eq. (\ref{eq:expect-error}). In fact, by the definition of $\mathrm{Err}$ (the first equality in Eq. (\ref{eq:expect-error})), the unlabelled sample $S'_{X_j'}$ does not explicitly appear, so one can easily apply the similar information-theoretic analysis starting from the first equality in Eq. (\ref{eq:expect-error}), and obtain an upper bound that consists of $I(W;Z_i)$ and $\mathrm{D_{KL}}(\mu||\mu')$. Precisely, the following bound holds,
 \begin{thm}
 \label{thm:itb-uda-notarget}
 Assume $\ell(f_w(X'),Y')$ is $R$-subgaussian for any $w\in\mathcal{W}$. Then 
 \[
    \left|{\mathrm{Err}}\right|
    \leq \frac{1}{n}\sum_{i=1}^n\mathbb{E}\sqrt{2R^2I(W;Z_i)}+\sqrt{2R^2\mathrm{D_{KL}}(\mu||\mu')}.
    \]
 \end{thm}
 The proof of Theorem~\ref{thm:itb-uda-notarget} is nearly the same to the proof of \cite[Corollary~ 2]{wu2020information} and \cite[
Corollary~1]{masiha2021learning}.

 It's important to note that although 
 \[I(W;Z_i)\leq I(W;Z_i|X_j')=\ex{X_j'}{I^{X_j'}(W;Z_i)},\] 
 the bound in Theorem~\ref{thm:ood-semi-bound-1} is incomparable to the bound based on $I(W;Z_i)$. This is mainly due to the fact that we use the disintegrated version of mutual information, $I^{X_j'}(W;Z_i)$, and the expectation over $X_j'$ is outside of the square root, which is a convex function. Using $I^{X_j'}(W;Z_i)$ instead of $I(W;Z_i)$ allows us to figure out more details about the role of unlabelled target data in the algorithm. Additionally, one can also prove a bound based on $I(W;Z_i|X_j')$ (e.g., simply applying Jensen's inequality  to Theorem~\ref{thm:ood-semi-bound-1}), which is close to an individual and UDA version of \cite[Theorem~3]{bu2022characterizing}.

Furthermore, the first term in Theorem~\ref{thm:ood-semi-bound-1} characterize the expected generalization gap on the source domain (i.e. $\ex{W,S}{R_\mu(W)-R_S(W)}$), then the bound suggests us that it's possible to invoke the unlabelled target data to further improve the performance on source domain, and the simplest case is the semi-supervised learning (when $\mu=\mu'$).

\paragraph{Compared with \citep{wu2020information,jose2021informationJS}.}
Notably, bounds in \citep{wu2020information,jose2021informationJS} fail to characterize the dependence between $W$ and $S'_{X'}$. More precisely, the algorithm-dependent term in their bounds is $I(W;Z_i)$ or $I(W;S)$, while our algorithm-dependent term is $I^{X_j'}(W;Z_i)$ that directly depends on the unlabelled target data. 
Moreover, while the disintegrated mutual information $I^{X_j'}(W;Z_i)$ and the unconditional mutual information $I(W;Z_i)$ cannot be directly compared, recent work by \cite{wang2023tighter} provides empirical evidence comparing similar terms in the supervised learning setting. Specifically, they demonstrate that when the empirical risk is small, such as in a realizable case, the disintegrated mutual information is smaller than the unconditional mutual information. Conversely, when the empirical risk is large, the unconditional mutual information is the smaller of the two.

%In addition, bounds in \cite{wu2020information,jose2021informationJS} fail to characterize the dependence between $W$ and $S'_{X'}$. More precisely, the algorithm-dependent term in their bounds is $I(W;Z_i)$ or $I(W;S)$, while our algorithm-dependent term is $I^{X_j'}(W;Z_i)$ that directly depends on the unlabelled target data (see Theorem~\ref{thm:itb-uda-notarget} for more discussion in Appendix).

\paragraph{More Discussion on the Vanishing of $I(X_j';Z_i|W)$ in Remark~\ref{rem:vanishing}.} 
Note that $S$ depends on $S_{X'}'$ given $W$, so intuitively the dependence between each  individual instance $Z_i$ and $X_j'$ is weaker when $n$ and $m$ become larger. More precisely, W.L.O.G let $i=j=1$, and recall that $W=\mathcal{A}(S_{X'}',S)$, when $n,m\rightarrow \infty$, taking $S$ and $S_{X'}'$ as the input of the algorithm is nearly equivalent to computing $W$ based on the source distribution $\mu$ and the target distribution $P_{X'}$, thus, $W$ will only depend on the two distributions, without depending on the realizations $Z_1$ and $X'_1$ drawn respectively from the two distributions, that is,  $I(Z_1;X_1'|W)=I(Z_1;X_1'|\mathcal{A}(\mu,P_{X'}))=I(Z_1;X_1')=0$. 
In addition, one may argue that what if $W=constant$ that does not really depend on the input data. In this case, $I(Z_1;X_1'|W)=I(Z_1;X_1')=0$ will hold trivially.
In the other extreme, if $n=1$ and $m=1$, then $W=\mathcal{A}(X_1',Z_1)$, and the quantity $I(Z_1;X_1'|\mathcal{A}(X_1',Z_1))$ should be large. When $n$ and $m$ increase, it becomes $I(Z_1;X_1'|\mathcal{A}(X_{1:m}',Z_{1:n}))$. Now we want to guess $Z_1$ from $X_1'$, this should be easier when having the knowledge of $\mathcal{A}(X_1';Z_1)$ compared with when having the knowledge of $\mathcal{A}(X_{1:m}';Z_{1:n})$.

\subsection{Proof of Theorem~\ref{thm:ood-semi-bound-1}}
\begin{proof}
    By Lemma~\ref{lem:DV-KL},
    \begin{align}
        &\mathrm{D_{KL}}\left(P_{W,Z_i|X_j'=x_j'}||P_{W,Z'|X_j'=x_j'}\right)\notag\\
        =&\mathrm{D_{KL}}\left(P_{W,Z_i|X_j'=x_j'}||P_{W|X_j'=x_j'}P_{Z'}\right)\label{eq:independence-1}\\
        \geq& \sup_t \ex{P_{W,Z_i|X_j'=x_j'}}{t\ell(f_W(X_i),Y_i)}-\log{\ex{P_{W|X_j'=x_j'}P_{Z'}}{\exp{\left(t\ell(f_W(X'),Y')\right)}}}\notag\\
        % \geq&\sup_t \ex{P_{W,Z_i|X_j'=x_j'}}{t\ell(f_W(X_i),Y_i)}-\ex{P_{W|X_j'=x_j'}}{tR_{\mu'}(W)}-\log{\ex{P_{W|X_j'=x_j'}P_{Z'}}{e^{t(\ell(f_W(X'),Y')-\ex{Z'}{\ell(f_W(X'),Y')})}}}\label{ineq:jensen-exp-1}\\
        \geq&\sup_t \ex{P_{W,Z_i|X_j'=x_j'}}{t\ell(f_W(X_i),Y_i)}-\ex{P_{W|X_j'=x_j'}}{tR_{\mu'}(W)}-R^2t^2/2,\notag
    \end{align}
    where Eq. (\ref{eq:independence-1}) is by the independence between algorithm output $W$ and unseen target domain data $Z'$, 
    % Eq. (\ref{ineq:jensen-exp-1}) is by Jensen's inequality for the exponential function 
    and the last inequality is by the subgaussian assumption.
    
    Thus,
    \begin{align}
        \left\lvert\ex{P_{W,Z_i|X_j'=x_j'}}{\ell(f_W(X_i),Y_i)}-\ex{P_{W|X_j'=x_j'}}{R_{\mu'}(W)}\right\rvert\leq\sqrt{2R^2 \mathrm{D_{KL}}\left(P_{W,Z_i|X_j'=x_j'}||P_{W|X_j'=x_j'}P_{Z'}\right)}.\label{ineq:single-kl}
    \end{align}

    Exploiting the fact that 
    \begin{align*}
        \left|{\mathrm{Err}}\right|=&\left|{\frac{1}{n}\sum_{i=1}^{n}\ex{{W,Z_i}}{\ell(f_W(X_i),Y_i)} - \ex{{W,Z'}}{\ell(f_W(X'),Y')}}\right|\\
        =&\left|\frac{1}{m}\sum_{j=1}^{m}\ex{X'_j}{\frac{1}{n}\sum_{i=1}^{n}\ex{{W,Z_i|X_j'}}{\ell(f_W(X_i),Y_i)} - \ex{{W,Z'|X_j'}}{\ell(f_W(X'),Y')}}\right|\\
        \leq&\frac{1}{m}\sum_{j=1}^{m}\mathbb{E}_{X_j'}\left|{\frac{1}{n}\sum_{i=1}^{n}\ex{{W,Z_i|X_j'}}{\ell(f_W(X_i),Y_i)} - \ex{{W,Z'|X_j'}}{\ell(f_W(X'),Y')}}\right|\\
        \leq& \frac{1}{nm}\sum_{j=1}^{m}\sum_{i=1}^n\mathbb{E}_{X_j'}\left|{\ex{{W,Z_i|X_j'}}{\ell(f_W(X_i),Y_i)} - \ex{{W|X_j'}}{R_{\mu'}(W)}}\right|,
    \end{align*}
    where the last two inequalities are by the Jensen's inequality for the absolute function.
    
    Notice that
    \begin{align*}
        \mathrm{D_{KL}}\left(P_{W,Z_i|X_j'=x_j'}||P_{W|X_j'=x_j'}P_{Z'}\right)
        =&\ex{P_{W,Z_i|X_j'=x_j'}}{\log{\frac{P_{W,Z_i|X_j'=x_j'}}{P_{W|X_j'=x_j'}P_{Z'}}}}\\
        =&\ex{P_{W,Z_i|X_j'=x_j'}}{\log{\frac{P_{W|Z_i,X_j'=x_j'}P_{Z_i}}{P_{W|X_j'=x_j'}P_{Z'}}}}\\
        =&\ex{P_{W,Z_i|X_j'=x_j'}}{\log{\frac{P_{W|Z_i,X_j'=x_j'}}{P_{W|X_j'=x_j'}}}}+\ex{P_{Z_i}}{\log{\frac{P_{Z_i}}{P_{Z'}}}}\\
        =&I(W;Z_i|X_j'=x_j')+\mathrm{D_{KL}}(\mu||\mu').
    \end{align*}
    
    Recall Eq. (\ref{ineq:single-kl}), we then have
    \begin{align*}
        \left|{\mathrm{Err}}\right|\leq&
        \frac{1}{nm}\sum_{j=1}^{m}\sum_{i=1}^n\mathbb{E}_{X_j'}\left|{\ex{{W,Z_i|X_j'}}{\ell(f_W(X_i),Y_i)} - \ex{{W|X_j'}}{R_{\mu'}(W)}}\right|\\
        \leq&\frac{1}{nm}\sum_{j=1}^{m}\sum_{i=1}^n\mathbb{E}_{X_j'}\sqrt{2R^2 \mathrm{D_{KL}}\left(P_{W,Z_i|X_j'}||P_{W|X_j'}P_{Z'}\right)}\\
        =&\frac{1}{nm}\sum_{j=1}^{m}\sum_{i=1}^n\mathbb{E}_{X_j'}\sqrt{2R^2(I^{X_j'}(W;Z_i)+\mathrm{D_{KL}}(\mu||\mu'))}\\
        \leq&\frac{1}{nm}\sum_{j=1}^{m}\sum_{i=1}^n\mathbb{E}_{X_j'}\sqrt{2R^2I^{X_j'}(W;Z_i)}+\sqrt{2R^2\mathrm{D_{KL}}(\mu||\mu')}.
    \end{align*}
    This completes the proof.
    \end{proof}

\subsection{Proof of Corollary~\ref{cor:bounded-mutual-lautum}}
\begin{proof}
    We now modify the proof in Theorem~\ref{thm:ood-semi-bound-1}.
    
    Recall that 
    \begin{align*}
        \left|{\mathrm{Err}}\right|\leq& \frac{1}{nm}\sum_{j=1}^{m}\sum_{i=1}^n\mathbb{E}_{X_j'}\left|{\ex{{W,Z_i|X_j'}}{\ell(f_W(X_i),Y_i)} - \ex{{W|X_j'}}{R_{\mu'}(W)}}\right|.
    \end{align*}
    
    We first decompose the right hand side,
    \begin{align*}
        &\left|{\ex{{W,Z_i|X_j'=x_j'}}{\ell(f_W(X_i),Y_i)} - \ex{{W|X_j'=x_j'}}{R_{\mu'}(W)}}\right|\\
        =& \left|{\ex{{W,Z_i|X_j'=x_j'}}{\ell(f_W(X_i),Y_i)}-\ex{{W|X_j'=x_j'}}{R_{\mu}(W)} +\ex{{W|X_j'=x_j'}}{R_{\mu}(W)} - \ex{{W|X_j'=x_j'}}{R_{\mu'}(W)}}\right|\\
        \leq& \left|\ex{{W,Z_i|X_j'=x_j'}}{\ell(f_W(X_i),Y_i)}-\ex{{W|X_j'=x_j'}}{R_{\mu}(W)}\right| +\left|\ex{{W|X_j'=x_j'}}{R_{\mu}(W)-R_{\mu'}(W)}\right|\\
        \leq& \left|\ex{{W,Z_i|X_j'=x_j'}}{\ell(f_W(X_i),Y_i)}-\ex{{W|X_j'=x_j'}}{R_{\mu}(W)}\right| +\frac{M}{\sqrt{2}}\sqrt{\min\{\mathrm{D_{KL}}(\mu||\mu'),\mathrm{D_{KL}}(\mu'||\mu)\}},
    \end{align*}
    where the last inequality is by Corollary~\ref{cor-bound-symmetric-kl}.
    
    Then for the first term in RHS, notice that
    \begin{align*}
        &\mathrm{D_{KL}}\left(P_{W,Z|X_j'=x_j'}||P_{W,Z_i|X_j'=x_j'}\right)\\
        =&\mathrm{D_{KL}}\left(P_{W|X_j'=x_j'}P_{Z}||P_{W,Z_i|X_j'=x_j'}\right)\\
        \geq& \sup_t \ex{P_{W|X_j'=x_j'}P_{Z}}{t\ell(f_W(X),Y)}-\log{\ex{P_{W,Z_i|X_j'=x_j'}}{\exp{t\ell(f_W(X_i),Y_i)}}}\\
        \geq&\sup_t \ex{P_{W|X_j'=x_j'}P_{Z}}{t\ell(f_W(X),Y)}-\ex{P_{W,Z_i|X_j'=x_j'}}{t\ell(f_W(X_i),Y_i)}\\&-\log{\ex{P_{W,Z_i|X_j'=x_j'}}{e^{t(\ell(f_W(X_i),Y_i)-\ex{P_{W,Z_i|X_j'=x_j'}}{\ell(f_W(X_i),Y_)})}}}\\
        \geq&\sup_t \ex{P_{W|X_j'=x_j'}}{tR_{\mu}(W)}-\ex{P_{W,Z_i|X_j'=x_j'}}{t\ell(f_W(X_i),Y_i)}-M^2t^2/8,
    \end{align*}
    where the last inequality is due to the fact that $\ell$ is bounded by $M$ and $\ell(f_W(X_i),Y_i)$ is $M/2$-subgaussian.
    
    Thus,
    \begin{align*}
        \left|\ex{{W,Z_i|X_j'=x_j'}}{\ell(f_W(X_i),Y_i)}-\ex{{W|X_j'=x_j'}}{R_{\mu}(W)}\right|\leq& \sqrt{\frac{M^2}{2}\mathrm{D_{KL}}\left(P_{W|X_j'=x_j'}P_{Z}||P_{W,Z_i|X_j'=x_j'}\right)}\\
        =& \sqrt{\frac{M^2}{2}L\left(W,Z_i|X_j'=x_j'\right)}.\\
    \end{align*}
    Plugging this inequality with the decomposition into the inequality at the beginning of the proof, we have
    \begin{align*}
        \left|{\mathrm{Err}}\right|\leq& \frac{1}{nm}\sum_{j=1}^{m}\sum_{i=1}^n\mathbb{E}_{X_j'}\sqrt{\frac{M^2}{2}L^{X_j'}\left(W,Z_i\right)}+\frac{M}{\sqrt{2}}\sqrt{\min\{\mathrm{D_{KL}}(\mu||\mu'),\mathrm{D_{KL}}(\mu'||\mu)\}}.
    \end{align*}
    Similar development also holds for $\mathrm{D_{KL}}\left(P_{W,Z_i|X_j'=x_j'}||P_{W|X_j'=x_j'}P_{Z}\right)$ as in the proof of Theorem~\ref{thm:ood-semi-bound-1}, thus
    \[
    \left|{\mathrm{Err}}\right|
    % \leq \frac{1}{nm}\sum_{i=1}^n\sum_{j=1}^m \sqrt{2\sigma^2(I(W;Z_i|X_j')+\mathrm{D_{KL}}(\mu||\mu'))}
    \leq \frac{M}{\sqrt{2}nm}\sum_{j=1}^{m}\sum_{i=1}^n\mathbb{E}_{X_j'}\sqrt{\min\left\{I^{X_j'}(W;Z_i),L^{X_j'}(W;Z_i)\right\}}+\frac{M}{\sqrt{2}}\sqrt{\min\left\{\mathrm{D_{KL}}(\mu||\mu'),\mathrm{D_{KL}}(\mu'||\mu)\right\}}.
    \]
    This completes the proof.
\end{proof}
    
\subsection{Proof of Theorem~\ref{thm:wasserstein-bound-expected}}
\begin{proof}
Similar to the proof of Corollary~\ref{cor:bounded-mutual-lautum},
recall Theorem~\ref{thm-bound-lipschitz},
    \begin{align*}
        &\left|{\mathrm{Err}}\right|\\
        \leq& \frac{1}{nm}\sum_{j=1}^{m}\sum_{i=1}^n\mathbb{E}_{X_j'}\left|{\ex{{W,Z_i|X_j'}}{\ell(f_W(X_i),Y_i)} - \ex{{W|X_j'}}{R_\mu(W)} + \ex{{W|X_j'}}{R_\mu(W)}- \ex{{W|X_j'}}{R_{\mu'}(W)}}\right|\\
         \leq& \frac{1}{nm}\sum_{j=1}^{m}\sum_{i=1}^n\mathbb{E}_{X_j'}\left|{\ex{{W,Z_i|X_j'}}{\ell(f_W(X_i),Y_i)} - \ex{{W|X_j'}}{R_\mu(W)} }\right|+\beta\mathbb{W}(\mu,\mu')\\
         \leq& \frac{1}{nm}\sum_{j=1}^{m}\sum_{i=1}^n\mathbb{E}_{X_j',Z_i}\left|{\ex{{W|Z_i,X_j'}}{\ell(f_W(X_i),Y_i)} - \ex{{W|X_j'}}{\ell(f_W(X_i),Y_i)} }\right|+\beta \mathbb{W}(\mu,\mu')\\
         \leq& \frac{\beta'}{nm}\sum_{j=1}^{m}\sum_{i=1}^n\mathbb{E}_{X_j',Z_i}\mathbb{W}(P_{W|X_j',Z_i},P_{W|X_j'})+\beta \mathbb{W}(\mu,\mu'),
    \end{align*}
where the last inequality is by Lemma~\ref{lem:KR duality}. 
This concludes the proof.
\end{proof}

\subsection{Proof of Corollary~\ref{cor:tv-bound-expected}}
\begin{proof}
    Similar to the proof of Corollary~\ref{cor-bound-total-variation-bounded}, replacing Wasserstein distance by the total variation and replacing $\beta$ and $\beta'$ by $M$, will give us the first inequality,
    \begin{align*}
        \left|\widetilde{\mathrm{Err}}\right|&\leq \frac{M}{nm}\sum_{j=1}^{m}\sum_{i=1}^n\ex{X_j',Z_i}{\mathrm{TV}(P_{W|Z_i,X_j'},P_{W|X_j'})}+M\mathrm{TV}(\mu,\mu').
    \end{align*}
    The second inequality is by Lemma~\ref{lem:pinsker},
    \[
       \left|\widetilde{\mathrm{Err}}\right|\leq \frac{M}{nm}\sum_{j=1}^{m}\sum_{i=1}^n\mathbb{E}_{X_j',Z_i}\sqrt{\frac{1}{2}\mathrm{D_{KL}}(P_{W|Z_i,X_j'}||P_{W|X_j'})}+\sqrt{\frac{M^2}{2}\mathrm{D_{KL}}(\mu||\mu')}.
    \]
    Again, one can also apply Lemma~\ref{lem:BH} here. This concludes the proof.
\end{proof}

%\section{Omitted Proofs in Section~\ref{sec:application}}

\subsection{Proof of Theorem~\ref{thm:uda-gradient-bound}}
\begin{proof}
    Recall Theorem~\ref{thm:ood-semi-bound-1} and by Jensen's inequality we have
    \begin{align*}
        \left|{\mathrm{Err}}\right|
    \leq& \frac{1}{nm}\sum_{j=1}^{m}\sum_{i=1}^n\mathbb{E}_{X_j'}\sqrt{2R^2I^{X_j'}(W;Z_i)}+\sqrt{2R^2\mathrm{D_{KL}}(\mu||\mu')}\\
    \leq&\sqrt{\frac{2R^2}{nm}\sum_{j=1}^{m}\sum_{i=1}^nI(W;Z_i|X_j')}+\sqrt{2R^2\mathrm{D_{KL}}(\mu||\mu')}.
    \end{align*}
    
    Let $X'_{1,\dots,j-1,j+1,\dots,m}=S'_{X'}\setminus X_j'$. Notice that 
    \begin{align*}
        I(W;Z_i|S'_{X'})=& I(W;Z_i|S'_{X'})+I(X'_{1,\dots,j-1,j+1,\dots,m};Z_i|X_j')\\
        =& I(W;Z_i|X_j')+I(X'_{1,\dots,j-1,j+1,\dots,m};Z_i|X_j',W)\\
        \geq& I(W;Z_i|X_j').
    \end{align*}

    Thus, $I(W;Z_i|X_j')\leq I(W;Z_i|S'_{X'})$. Then
    \[
    \frac{1}{nm}\sum_{j=1}^{m}\sum_{i=1}^nI(W;Z_i|X_j')\leq\frac{1}{nm}\sum_{j=1}^{m}\sum_{i=1}^nI(W;Z_i|S'_{X'})=\frac{1}{n}\sum_{i=1}^nI(W;Z_i|S'_{X'}).
    \]
    
    Then, since $S \perp\!\!\!\perp S'_{X'}$ and $Z_i\perp\!\!\!\perp Z_{1:i-1}$ for any $i\in [n]$, by the chain rule of mutual information, we have
    \begin{align*}
        I(W;S|S'_{X'})=\sum_{i=1}^nI(W;Z_i|S'_{X'},Z_{1:i-1})=&\sum_{i=1}^nI(W;Z_i|S'_{X'},Z_{1:i-1})+I(Z_i;Z_{1:i-1})\\
        =&\sum_{i=1}^nI(W,Z_{1:i-1};Z_i|S'_{X'})\\
        =&\sum_{i=1}^nI(W;Z_i|S'_{X'})+I(Z_i;Z_{1:i-1}|S'_{X'},W)\\
        \geq&\sum_{i=1}^nI(W;Z_i|S'_{X'}).
    \end{align*}
    
    Thus, the generalization error bound becomes
    \[
    \left|{\mathrm{Err}}\right|\leq\sqrt{\frac{2R^2}{n}I(W;S|S'_{X'})}+\sqrt{2R^2\mathrm{D_{KL}}(\mu||\mu')}.
    \]
    
    Recall the updating rule of $W$ and notice that $W_0$ is independent of $S$ and $S'_{X'}$, the following process is by using the chain rule of mutual information and data processing inequality recurrently,
    \begin{align*}
        I(W_T;S|S'_{X'})=& I(W_{T-1}-\eta_T g(W_{T-1},Z_{B_T},X'_{B_T})+N_T;S|S'_{X'})\\
        \leq& I(W_{T-1},-\eta_T g(W_{T-1},Z_{B_T},X'_{B_T})+N_T;S|S'_{X'})\\
        =& I(W_{T-1};S|S'_{X'})+I(\eta_T g(W_{T-1},Z_{B_T},X'_{B_T})+N_T;S|S'_{X'},W_{T-1})\\
        \vdots&\\
        =&\sum_{t=1}^T I(\eta_t g(W_{t-1},Z_{B_t},X'_{B_t})+N_t;S|S'_{X'},W_{t-1}).
    \end{align*}
    
    For each $t\in [T]$, denote $g(W_{t-1},Z_{B_t},X'_{B_t})$ as $G_t$, then
    
    \begin{align*}
        I(\eta_t g(W_{t-1},Z_{B_t},X'_{B_t})+N_t;S|S'_{X'},W_{t-1})=&\ex{S'_{X'},W_{t-1},S}{\mathrm{D_{KL}}(P_{G_t+\frac{N_t}{\eta_t}|S,S'_{X'},W_{t-1}}||P_{G_t+\frac{N_t}{\eta_t}|S'_{X'},W_{t-1}})}\\
        \leq& \ex{S'_{X'},W_{t-1},S}{\mathrm{D_{KL}}(P_{G_t+\frac{N_t}{\eta_t}|S,S'_{X'},W_{t-1}}||P_{\ex{S}{G_t}+\frac{N_t}{\eta_t}|S'_{X'},W_{t-1}})}\\
        =&\frac{\eta_t^2}{2\sigma^2_t}\ex{S'_{X'},W_{t-1},S}{\left|\left|G_t-\ex{S}{G_t}\right|\right|^2},
    \end{align*}
    where the inequality is by Lemma~\ref{lem:mi-center-gravity} and the last equality is by the KL divergence between two Gaussian distributions.
    
    Finally, putting everything together, 
    \[
    \left|{\mathrm{Err}}\right|\leq\sqrt{\frac{R^2}{n}\sum_{t=1}^T\frac{\eta_t^2}{\sigma^2_t}\ex{S'_{X'},W_{t-1},S}{\left|\left|G_t-\ex{S}{G_t}\right|\right|^2}}+\sqrt{2R^2\mathrm{D_{KL}}(\mu||\mu')},
    \]
    which concludes the proof.
    \end{proof}

\subsection{Derivation of Eq.~(\ref{eq:cross-entropy decomposition})}

% \begin{proof}
    Recall the expected cross-entropy loss, we have
    \begin{align*}
        \ex{W,Z_i}{\ell(f_W(T_i),Y_i)} &= \ex{Z_i,W}{-\log{Q_{Y_i|T_i,W}}}\\
        &=\ex{Z_i,W}{\log{\frac{P_{Y_i|T_i,W}}{Q_{Y_i|T_i,W}P_{Y_i|T_i,W}}}}\\
        &=H(Y_i|T_i,W)  +\ex{X_i,W}{\mathrm{D_{KL}}(P_{Y_i|T_i,W}||Q_{Y_i|T_i,W})}\\
        &=\ex{Z_i,W}{\log{\frac{P_{Y_i|T_i}P_{W|T_i}}{P_{Y_i|T_i,W}P_{Y_i|T_i}P_{W|T_i}}}} +\ex{T_i,W}{\mathrm{D_{KL}}(P_{Y_i|T_i,W}||Q_{Y_i|T_i,W})}\\
        &=\ex{Z_i,W}{\log{\frac{P_{Y_i|T_i}P_{W|T_i}}{P_{Y_i,W|T_i}P_{Y_i|T_i}}}}  +\ex{T_i,W}{\mathrm{D_{KL}}(P_{Y_i|T_i,W}||Q_{Y_i|T_i,W})}\\
        &=H(Y_i|T_i)-I(W;Y_i|T_i)+\ex{T_i,W}{\mathrm{D_{KL}}(P_{Y_i|T_i,W}||Q_{Y_i|T_i,W})}
    \end{align*}
    % \end{proof}
    
\subsection{Additional Discussion on LIMIT}
\label{sec:limit}
In Section~\ref{sec:expected-bound}, we discussed the LIMIT approach proposed by \cite{harutyunyan2020improving} as a means of controlling label information memorization during training. Roughly speaking, to update the classifier parameters, LIMIT constructs an auxiliary network that predicts gradients instead of using the true gradients, which avoids direct use of the true labels for training. To obtain accurate gradients, the auxiliary network needs to be trained using the true labels. We found that the training of LIMIT is unstable and difficult to tune the hyperparameters when used under UDA settings. Therefore, we opted to use the pseudo label strategy proposed in Section~\ref{sec:expected-bound} instead of the pseudo gradient strategy.

% As mentioned in Section~\ref{sec:expected-bound},  \cite{harutyunyan2020improving} proposed an approach called LIMIT, refers to limiting label information memorization in training, to control label information.
  % Roughly speaking, to update the parameters of the classifier, they construct an auxiliary network to predict the gradient instead of using the real gradient, in which case the true label is not directly used for training the classifier. To provide accurate gradients, they also need to train the auxiliary network by using the true labels. We find that the training of LIMIT is unstable and hard to tune the hyperparameters under the UDA setting. Thus, we choose to use the pseudo label strategy proposed in Section~\ref{sec:expected-bound} instead of pseudo gradient strategy.

\section{Experiment Details}
We implemented our approach using PyTorch \citep{paszke2019pytorch} and conducted all experiments on NVIDIA Tesla V100 GPUs with 32 GB of memory. Our code builds largely on the implementation from \cite{gulrajani2021in}\footnote{Available at: \href{https://github.com/facebookresearch/DomainBed}{https://github.com/facebookresearch/DomainBed}.} and \cite{nguyen2022kl}\footnote{Available at: \href{https://github.com/atuannguyen/kl}{https://github.com/atuannguyen/kl}.}.
% The implementation in this paper is on PyTorch  \citep{paszke2019pytorch}, and all the experiments are carried out on NVIDIA Tesla V100 GPUs (32 GB). The code of all the experiments is largely based on the code of \cite{gulrajani2021in}\footnote{Available at: \href{https://github.com/facebookresearch/DomainBed}{https://github.com/facebookresearch/DomainBed}.} and \cite{nguyen2022kl}\footnote{Available at: \href{https://github.com/atuannguyen/kl}{https://github.com/atuannguyen/kl}.}.

\subsection{Objective Functions of Gradient Penalty and Controlling Label Information}
For every iteration, the objective function after adding the gradient penalty becomes 
\[
\min_{W} \hat{L}(W,Z_{B_t},X'_{B_t})+\lambda_1||g(W,Z_{B_t},X'_{B_t})||^2,
\]
where $\hat{L}(W,Z_{B_t},X'_{B_t})$ is some loss function for the source and target domain data in the current mini-batch and $\lambda_1$ is the trade-off coefficient. For example, if we combine ERM with gradient penalty then $\hat{L}(W,Z_{B_t},X'_{B_t})=\frac{1}{|B_t|}\sum_{k\in B_t}\ell(f_W(X_k),Y_k)$ and $\ell$ could be the cross-entropy loss. Moreover, if we combine KL guided marginal alignment algorithm \citep{nguyen2022kl} with gradient penalty then the objective function is
\[
\min_{W,\theta} \frac{1}{|B_t|}\sum_{k\in B_t}\ell(f_W(T_k),Y_k)+\beta_1\mathrm{D_{KL}}(P_{T'}||P_T)+\beta_2\mathrm{D_{KL}}(P_T||P_{T'})+\lambda_1||g(W,Z_{B_t},X'_{B_t})||^2,
\]
where $\theta$ is the parameters of the representation network and the gradient is
\[
g(W,Z_{B_t},X'_{B_t})= \frac{1}{|B_t|}\sum_{k\in B_t}\nabla_{W,\theta}\ell(f_W(T_k),Y_k)+\beta_1\nabla_\theta\mathrm{D_{KL}}(P_{T'}||P_T)+\beta_2\nabla_\theta\mathrm{D_{KL}}(P_T||P_{T'}).
\]

In \cite{nguyen2022kl}, the representation distribution is modelled as an Gaussian distribution, i.e., $T\sim \mathcal{N}(\mu_\theta, \sigma^2_\theta \mathrm{I}_d|X)$ and $T'\sim \mathcal{N}(\mu_\theta, \sigma^2_\theta \mathrm{I}_d|X')$. Additionally, let the batch size be $b=|B_t|$, the empirical KL divergence is estimated by the mini-batch data, as given in \cite{nguyen2022kl},
\begin{align*}
    &\beta_1\mathrm{D_{KL}}(P_{T'}||P_T)+\beta_2\mathrm{D_{KL}}(P_T||P_{T'}) \\
    \approx& \beta_1\frac{1}{b}\sum_{k\in B_t} [\log{P_{T_k'}}-\log{P_{T_k}}]+\beta_2\frac{1}{b}\sum_{k\in B_t} [\log{P_{T_k}}-\log{P_{T'_k}}]\\
    \approx& \beta_1\frac{1}{b}\sum_{k\in B_t} \left[\log{\frac{1}{b}\sum_{k\in B_t} P_{T'_k|X'_k}}-\log{\frac{1}{b}\sum_{k\in B_t} P_{T_k|X_k}}\right]+\beta_2\frac{1}{b}\sum_{k\in B_t} \left[\log{{\frac{1}{b}\sum_{k\in B_t} P_{T_k|X_k}}}-\log{\frac{1}{b}\sum_{k\in B_t} P_{T'_k|X'_k}}\right],
\end{align*}
where $P_{T_k|X_k}=\mathcal{N}(\mu_\theta, \sigma^2_\theta \mathrm{I}_d|X_k)$ and $P_{T'_k|X'_k}=\mathcal{N}(\mu_\theta, \sigma^2_\theta \mathrm{I}_d|X'_k)$. To be more precise, $\mu_\theta$ and $\sigma_\theta$ are the outputs of the representation network. Since the forward pass requires the sampling of $T_k$ and $T_k'$, we need to use the reparameterization trick \citep{kingma2013auto} for the backward pass.

When we train the model with controlling label information, the objective function becomes
\[
\min_{W} \hat{L}(W,Z_{B_t},X'_{B_t})+\lambda_2||W-\widetilde{W}||^2,
\]
where $\widetilde{W}$ is the auxiliary classifier and $\lambda_2$ is the trade-off hyperparameter.

Similarly, when we combine KL guided marginal alignment algorithm with controlling label information, then the objective function in every iteration is
\[
\min_{W,\theta} \frac{1}{|B_t|}\sum_{k\in B_t}\ell(f_W(T_k),Y_k)+\beta_1\mathrm{D_{KL}}(P_{T'}||P_T)+\beta_2\mathrm{D_{KL}}(P_T||P_{T'})+\lambda_2||W-\widetilde{W}||^2.
\]
In addition, the training objective for the auxiliary classifier is
\begin{align}
\label{eq:cl-obj}
    \min_{\widetilde{W}} \frac{1}{|B_t|}\sum_{k\in B_t}\ell(f_{\widetilde{W}}(T'_k), f_W(T'_k))+\frac{1}{|B_t|}\sum_{k\in B_t}\ell(f_{\widetilde{W}}(T_k), f_W(T_k)).
\end{align}
In practice, removing the second term would not affect the performance. Note that we need to disenable the automatic differentiation of $T$, $T'$ and $W$ when executing the backward pass for the auxiliary classifier. The detailed algorithm of controlling label information is given in the next section.

\subsection{Algorithm of Controlling Label Information and Additional Results of {ERM-CL}}
\begin{algorithm}[!htbp]
\caption{Controlling Label Information}
\label{alg:cl}
\begin{algorithmic}[2]
\REQUIRE Source domain labelled dataset $S$, Target domain unlabelled dataset $S'_{X'}$, Batch size $b$, Classification loss function $\ell_c$, Marginal alignment loss function $\ell_r$, Initial classifier parameter $\boldsymbol{w}_0=\widetilde{\boldsymbol{w}}_0$, Initial representation network parameter $\boldsymbol{\theta}_0$, Learning rate $\eta$, Lagrange multiplier $\lambda_2$ 
\WHILE{$\boldsymbol{w}_t,\theta_t$ not converged}
\STATE Update iteration: $t\gets t+1$
\STATE Sample $\mathcal{Z_B}=\{\boldsymbol{z}_i\}_{i=1}^b$ from source domain training set $S$
\STATE Sample $\mathcal{X'_B}=\{\boldsymbol{x}'_i\}_{i=1}^b$ from target domain training set $S'_{X'}$
\STATE Compute distance from the auxiliary classifier $\boldsymbol{dis}\gets||\boldsymbol{w}_t-\widetilde{\boldsymbol{w}}_t||^2$ 
\STATE Compute marginal alignment loss $L_r \gets \frac{1}{b} \sum_{i=1}^b \ell_r(\theta_t,\boldsymbol{z}_i, \boldsymbol{x}'_i)$
\STATE Compute classification loss $L_c \gets \frac{1}{b} \sum_{i=1}^b \ell_c(\boldsymbol{w}_{t}, \theta_t,\boldsymbol{z}_i, \boldsymbol{x}'_i)$
\STATE Compute gradient: \\ 
$g_{\mathcal{B}} \gets \nabla (L_c+L_r+\lambda_2 \boldsymbol{dis})$
\STATE Update parameter: $\boldsymbol{w}_{t+1}\gets\boldsymbol{w}_{t}-\eta\cdot g_{\mathcal{B}}$, $\boldsymbol{\theta}_{t+1}\gets\boldsymbol{\theta}_{t}-\eta\cdot g_{\mathcal{B}}$
\STATE Obtain the pseudo labels $\mathcal{Y'_B}\gets f_{\boldsymbol{w}_{t}}(g_{\boldsymbol{\theta}_{t}}(\mathcal{X'_B}))$
\STATE Compute auxiliary classification loss $L_a \gets \frac{1}{b} \sum_{i=1}^b \ell_c(\widetilde{\boldsymbol{w}}_t, \theta_t,\boldsymbol{x}'_i, \boldsymbol{y}'_i)$
\STATE Compute auxiliary classifier gradient: \\ 
$\widetilde{g}_{\mathcal{B}} \gets \nabla L_a$
\STATE Update auxiliary classifier parameter: $\widetilde{\boldsymbol{w}}_{t+1}\gets\widetilde{\boldsymbol{w}}_{t}-\eta\cdot \widetilde{g}_{\mathcal{B}}$
\ENDWHILE
\end{algorithmic}
\end{algorithm}
If we only provide the pseudo labels for the target domain data to the auxiliary classifier, i.e. removing the second term in Eq~(\ref{eq:cl-obj}), the Algorithm~\ref{alg:cl} is the algorithm for combining any marginal alignment algorithm with controlling label information.

Even without incorporating with the marginal alignment algorithm, e.g., ERM, in which case $L_r$ is removed,  Algorithm~\ref{alg:cl} still boosts the performance in practice. 

\begin{table*}[t!]
%  \footnotesize
 \scriptsize
 \centering
 \caption{RotatedMNIST and Digits Experiments of \textbf{ERM-CL}. Results of ERM are reported from \cite{nguyen2022kl}.}
%  \vspace{-0.08in}
  \label{tab:RM-Digits-ERMCL}
 \centering
 	\begin{tabular}{ccccccccccc}
 		\toprule
 		& \multicolumn{6}{c}{RotatedMNIST ($\mathbf{0^{\circ}}$ as source domain)} & \multicolumn{4}{c}{Digits}\\
		\cmidrule(r){2-7}
		\cmidrule(r){8-11}
		Method  & $\mathbf{15^{\circ}}$  & $\mathbf{30^{\circ}}$   & $\mathbf{45^{\circ}}$ & $\mathbf{60^{\circ}}$ & $\mathbf{75^{\circ}}$ & \textbf{Ave} & \textbf{M $\rightarrow$ U}  & \textbf{U $\rightarrow$ M}  & \textbf{S $\rightarrow$ M}  & \textbf{Ave}   \\
		\midrule
		ERM & 97.5±0.2 & 84.1±0.8 & 53.9±0.7 & 34.2±0.4 & 22.3±0.5 & 58.4 & 73.1±4.2 & 54.8±6.2 & 65.9±1.4 & 64.6 \\
		\midrule
		ERM-GP & \textbf{97.5±0.1} & \textbf{86.2±0.5} & \textbf{62.0±1.9} & \textbf{34.8±2.1} & \textbf{26.1±1.2} & \textbf{61.2} & \textbf{91.3±1.6} & \textbf{72.7±4.2} & {68.4±0.2} &  77.5 \\
		ERM-CL  & {97.3±0.1} & {84.1±0.1} & {56.9±2.5} & {34.2±1.9} & 25.5±1.6 & 59.6 & 88.9±0.4& 71.2±3.6 & \textbf{73.5±1.4} & \textbf{77.9} \\
 		\bottomrule
 	\end{tabular}
\end{table*}

Table~\ref{tab:RM-Digits-ERMCL} shows that \textbf{ERM-CL} can overall outperform the basic \textbf{ERM} and is close to the performance of \textbf{ERM-GP}.

\subsection{Architectures and Hyperparameters}
The network architecture in this work is the same as in \cite{gulrajani2021in} and \cite{nguyen2022kl}, where a simple CNN is used.

Other settings are also the same as \cite{gulrajani2021in} and \cite{nguyen2022kl}, for example, each algorithm is trained for $100$ epochs. To select the hyperparameters ($\lambda_1$ and $\lambda_2$) for \textbf{ERM-GP}, \textbf{ERM-KL}, \textbf{KL-GP} and \textbf{KL-CL}, we perform random search. Specifically, $\lambda_1$ is searched between $[0.1, 0.9]$ and $\lambda_2$ is searched between $[10^{-6}, 0.8]$. Other hyperparameters searching range could be found in the source code of \cite{nguyen2022kl}.

\subsection{Additional Experimental Results}
\subsection{Ablation Study on the Effect of Gradient Penalty Hyperparameter}
Our study includes an ablation analysis to investigate the impact of the hyperparameter $\lambda_1$ in the context of \textbf{KL-GP}. Specifically, we conduct experiments on both RotatedMNIST and Digits datasets, where the source and target domains are set to 0°/60° and SVHN/MNIST, respectively. Table~\ref{tab:ablation-study} summarizes the results. It is worth noting that setting $\lambda_1$ to zero effectively reduces \textbf{KL-GP} to \textbf{KL}, and our results confirm the efficacy of including the gradient penalty term in \textbf{KL-GP}.
% We provide an ablation study on the value of $\lambda_1$ when using \textbf{KL-GP}. In the RotatedMNIST experiment, we take $\mathbf{0^{\circ}}$ as the source domain and take $\mathbf{60^{\circ}}$ as the target domain. In the Digits experiment, we take SVHN as the source and take MNIST as the target. The results are shown in Table~\ref{tab:ablation-study}. Notice that when $\lambda_1=0$, \textbf{KL-GP} becomes \textbf{KL}, and the results again justify the effeteness of adding gradient penalty term.
\begin{table*}[t!]
%  \footnotesize
%  \centering
 \centering
 \caption{{Ablation study on Effect of Gradient Penalty Hyperparameter}.}
%  \vspace{-0.08in}
 	\begin{tabular}{c|cccc}
 		\toprule
		  $\lambda_1$ & 0  & 0.1   & 0.3 & 0.5  \\
		\midrule
		$\mathbf{0^{\circ}}\rightarrow\mathbf{60^{\circ}}$ & 75.5±2.4 & 88.0±8.1 & 82.8±5.8 & 80.1±3.7 \\
		$\textbf{S $\rightarrow$ M}$ & 92.5±0.9 & 93.6±1.2 & 93.8±1.1 & 93.1±1.7 \\
 		\bottomrule
 	\end{tabular}
  \label{tab:ablation-study}
%   \vspace{-0.13in}
\end{table*}

\subsection{Visualization Results}

\begin{figure}[htbp]
    \centering
    % \subfloat[Visualization of KL]{%
    %   \includegraphics[width=0.34\linewidth]{data/KL_source_target.pdf}}
    % \subfloat[Visualization of KL-GP]{%
    %   \includegraphics[width=0.34\linewidth]{data/KLGP_source_target.pdf}}
    %  \subfloat[Visualization of KL-CL]{%
    %   \includegraphics[width=0.34\linewidth]{data/KLCL_source_target.pdf}}
       \begin{subfigure}[t]{0.32\columnwidth}
           \centering
          \includegraphics[scale=0.16]{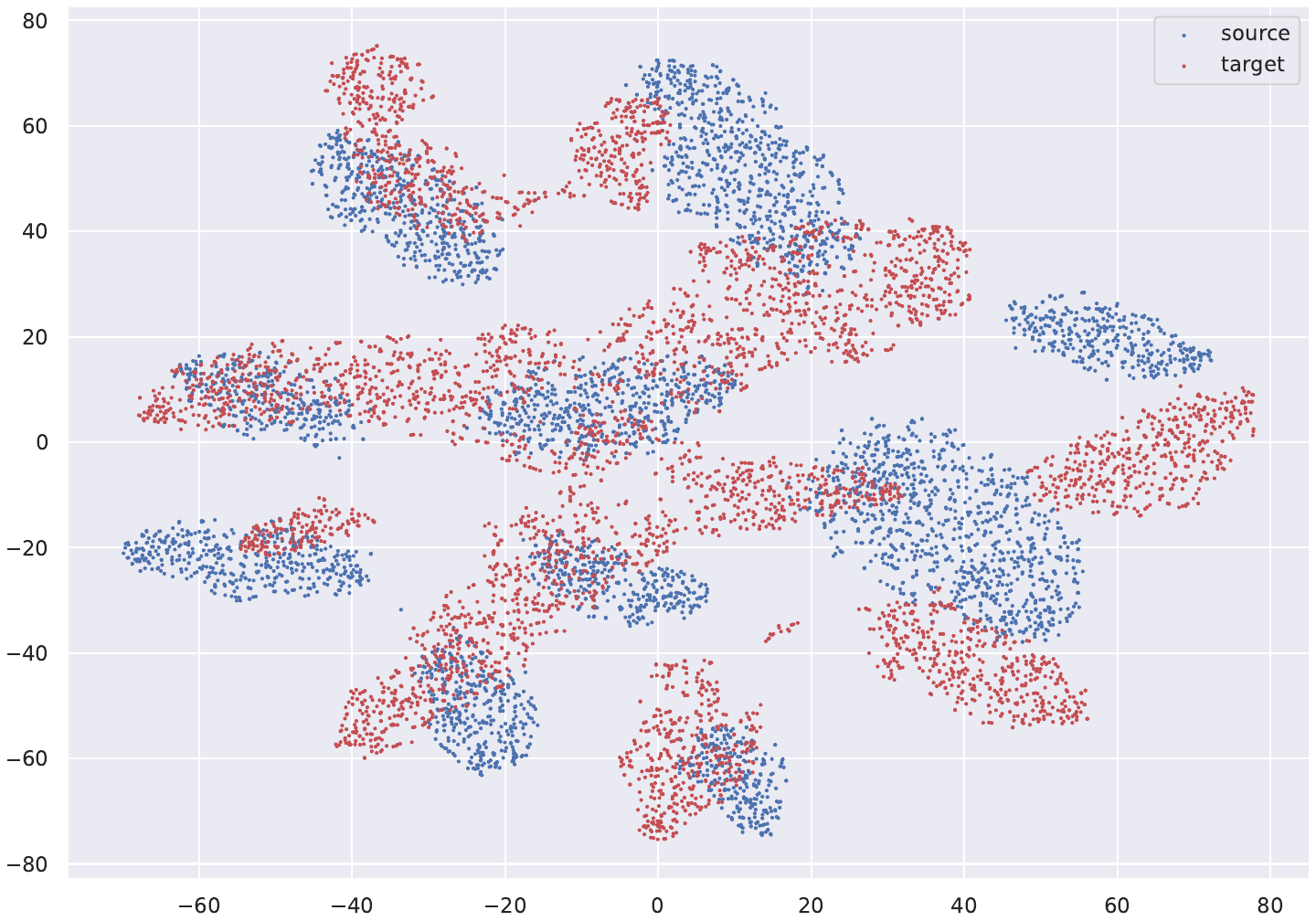}
            \caption{Visualization of KL}
            \label{fig:a}
    \end{subfigure}
    \begin{subfigure}[t]{0.32\columnwidth}
            \centering
            \includegraphics[scale=0.16]{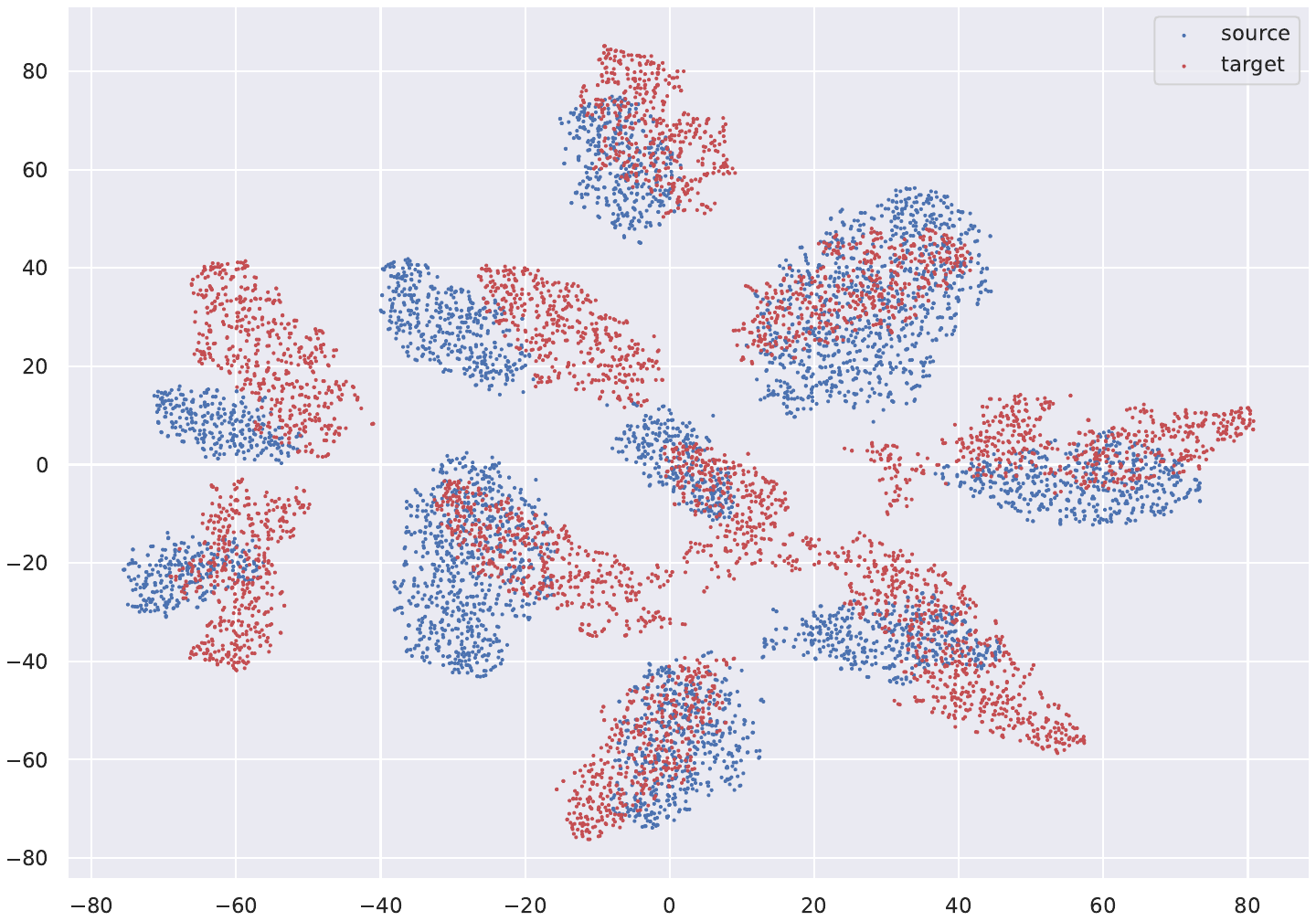}
            \caption{Visualization of KL-GP}
            \label{fig:b}
    \end{subfigure}
\begin{subfigure}[t]{0.32\columnwidth}
            \centering
            \includegraphics[scale=0.16]{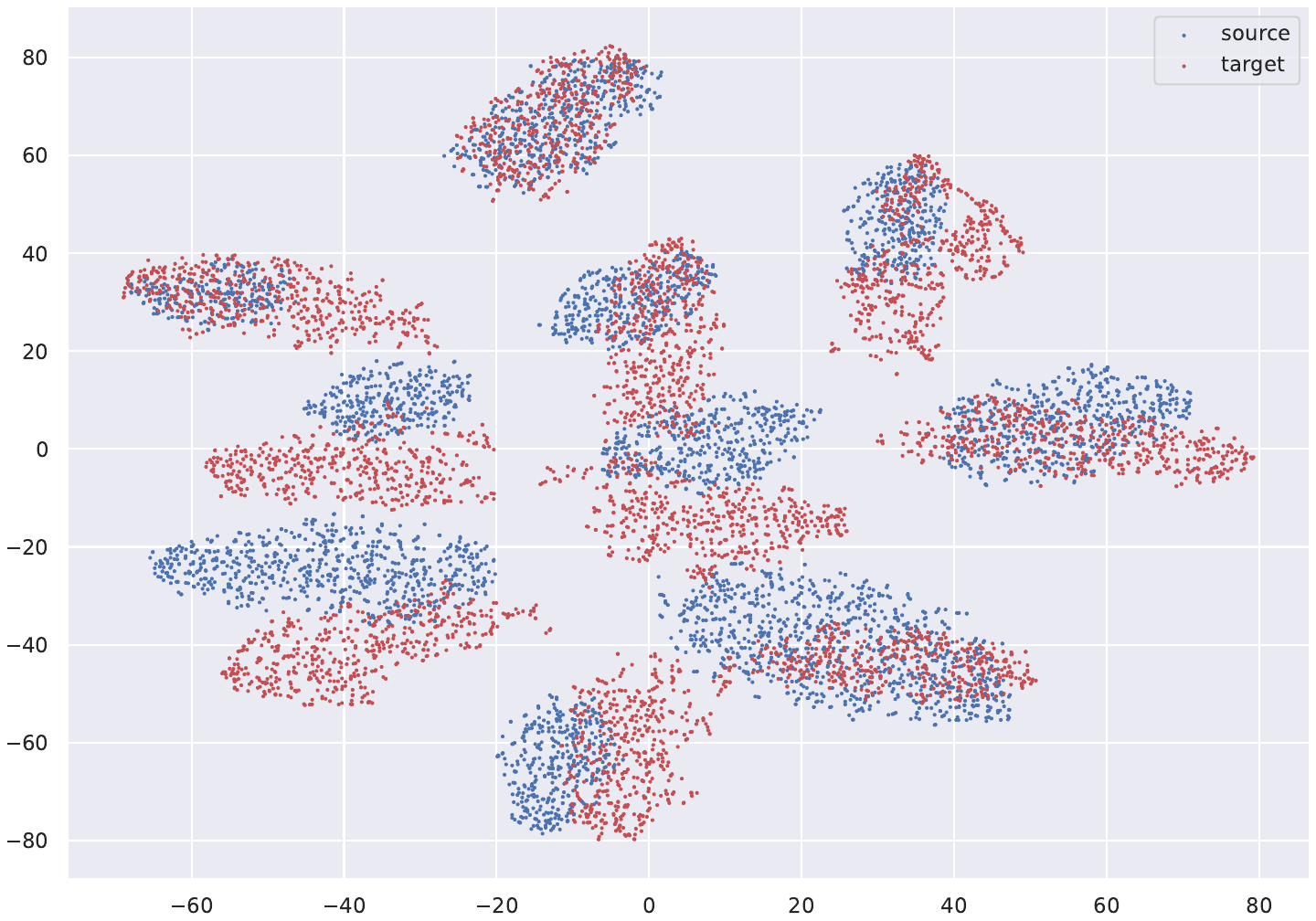}
            \caption{Visualization of KL-CL}
            \label{fig:c}
    \end{subfigure}
    \caption{Visualization results of representations obtained by using t-SNE. The source domain (blue points) is SVHN and the target domain (red points) is MNIST.}
    \label{fig:vis-tsne}
\end{figure}
To visualize the representations of models trained using \textbf{KL}, \textbf{KL-GP}, and \textbf{KL-CL}, we employ t-SNE \citep{van2008visualizing}. Figure~\ref{fig:vis-tsne} displays the visualization results when SVHN is used as the source domain and MNIST as the target domain. Our findings indicate that the incorporation of additional regularizers yields a slight improvement in representation alignment. However, it is essential to note that these regularization terms are primarily designed to enhance the performance of the classifier network, rather than the representation network.
% We visualize the representations of models trained with \textbf{KL}, \textbf{KL-GP} and \textbf{KL-CL} by using t-SNE \citep{van2008visualizing}, and the visualization results are shown in Figure~\ref{fig:vis-tsne}. Here we take SVHN as the source domain and take MNIST as the target domain.  We notice that with the help of additional regularizers, the representation alignment could be slightly improved. It is still important to note that the regularization terms are mainly designed for improving the performance of the classifier network instead of the representation network.

\subsection{Results on VisDA17}
\begin{table*}[t!]
%  \footnotesize
 \centering
 \caption{VisDA17 experiments. Results of baselines are reported directly from \cite{nguyen2022kl}.}
%  \vspace{-0.08in}
  \label{tab:visda17}
 	\begin{tabular}{c|c}
 		\toprule
 		 %& VisDA17\\
		Method   & Synthetic $\rightarrow$ Real  \\
		\midrule
		ERM & 39.1±0.5\\
		DANN & 57.7±1.3\\
		MMD & 62.8±1.1\\
		CORAL & 39.5±4.5 \\
		WD & 38.9±4.8\\
		KL & 70.6±0.5\\
		\midrule
		KL-GP & \textbf{71.9±0.7}\\
		KL-CL & {71.3±0.4}\\
 		\bottomrule
 	\end{tabular}
\end{table*}
We also conduct experiments on the VisDA17 dataset \citep{peng2017visda}, which is a real-world classification task with $280K$ images from $12$ classes. Particularly, the source domain
contains synthetic images and the target domain contains real images. Table~\ref{tab:visda17} presents our experimental results. Notably, our regularization techniques, namely \textbf{KL-GP} and \textbf{KL-CL} are still capable of improving the performance of the KL guided marginal alignment algorithm to some extent.
% still able to boost the performance of the KL guided marginal alignment algorithm.

\begin{figure}[ht!]
% \vspace{-5pt}
\input{klplot}
\caption{\textbf{KL} on \textbf{S$\rightarrow$M}. The left figure is the comparison of the Jeffrey's divergence in the representation space and the testing error. The right figure is the evolution of testing accuracy with respect to the different fraction of unlabelled target data used for training.
}
\label{fig:train-dynamic}
% \vspace{-5pt}
\end{figure}
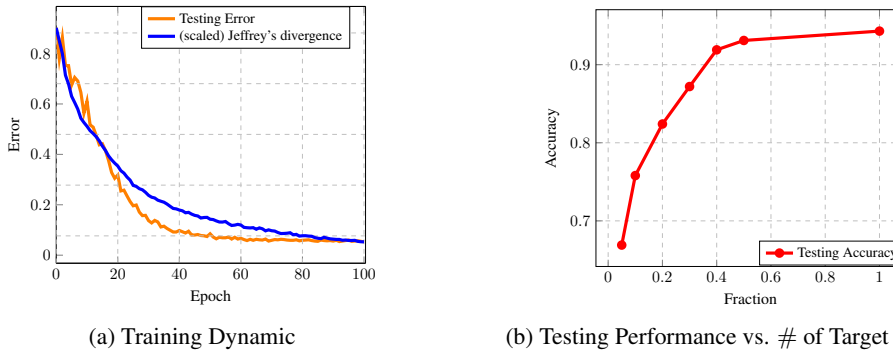

\subsection{Dynamics of Jeffrey's divergence}
The representation space version of Corollary~\ref{cor-bound-symmetric-kl} suggests that a small Jeffrey's divergence can lead to a low testing error. Figure~\ref{fig:svhn-mnist} demonstrates that the dynamic of Jeffrey's divergence, as computed in the representation space, can effectively characterize the evolution of the testing error throughout the training phase. Additionally, Figure~\ref{fig:sample} reveals that the number of target data used has an impact on testing performance. Specifically, when less than half of the available unlabelled target data is used, performance increases with the number of data. However, when more than half of the unlabelled target data is used, there is only marginal improvement on performance.

%% file: klplot.tex
\begin{subfigure}[t]{0.5\columnwidth}%
\centering%
\captionsetup{font=small}%
\scalebox{0.6}{
\begin{tikzpicture}
\begin{axis}
[
% legend style={at={(1,0.8)},nodes={scale=0.9, transform shape}},
% ymode=log,
axis y line*=left,
xlabel=Epoch,
% table/col sep=comma,
grid style=dashed,
xmin=0,
    xmax=100.5,
ylabel=Error
]
\addplot[line width=2pt,  color=orange] table [y=Error, x=epoch]{data/KL.txt};
\label{Test}
\end{axis}
\begin{axis}
[legend style={at={(0.95,1)},nodes={scale=0.9, transform shape}},
legend cell align={left},
axis y line*=right,
scaled y ticks = false,
ticks=none,
xmin=0,
    xmax=100.5,
% table/col sep=comma,
grid style=dashed,
xmajorgrids=true,
ymajorgrids=true,
ylabel style={rotate=-180},
% axis background/.style={fill=blue!10}
%   ylabel=Norm
]
\addlegendimage{/pgfplots/refstyle=Test}\addlegendentry{Testing Error}
\addplot[line width=2pt,  color=blue] table [y=KL, x=epoch]{data/KL.txt};
\addlegendentry{(scaled) Jeffrey’s divergence}
\end{axis}
\end{tikzpicture}
}
\caption{Training Dynamic}%
\label{fig:svhn-mnist}
\end{subfigure}
\hfill
\begin{subfigure}[t]{0.5\columnwidth}%
\centering%
\captionsetup{font=small}%
\scalebox{0.6}{
\begin{tikzpicture}
\begin{axis}
[legend style={at={(1,0.1)},nodes={scale=0.9, transform shape}},
% ymode=log,
% axis y line*=left,
% ymin=0.071, ymax=0.35,
xlabel=Fraction,
% table/col sep=comma,
grid style=dashed,
xmajorgrids=true,
ymajorgrids=true,
ylabel=Accuracy
]
\addplot[line width=2pt, mark=*, color=red] table [y=acc, x=frac]{data/sample.txt};
\addlegendentry{Testing Accuracy}
\end{axis}
\end{tikzpicture}
}
\caption{Testing Performance vs. $\#$ of Target Data}%
\label{fig:sample}
\end{subfigure}%

%% file: main.bbl
\begin{thebibliography}{77}
\providecommand{\natexlab}[1]{#1}
\providecommand{\url}[1]{\texttt{#1}}
\expandafter\ifx\csname urlstyle\endcsname\relax
  \providecommand{\doi}[1]{doi: #1}\else
  \providecommand{\doi}{doi: \begingroup \urlstyle{rm}\Url}\fi

\bibitem[Achille \& Soatto(2018)Achille and Soatto]{achille2018emergence}
Alessandro Achille and Stefano Soatto.
\newblock Emergence of invariance and disentanglement in deep representations.
\newblock \emph{The Journal of Machine Learning Research}, 19\penalty0
  (1):\penalty0 1947--1980, 2018.

\bibitem[Acuna et~al.(2021)Acuna, Zhang, Law, and Fidler]{acuna2021f}
David Acuna, Guojun Zhang, Marc~T Law, and Sanja Fidler.
\newblock f-domain adversarial learning: Theory and algorithms.
\newblock In \emph{International Conference on Machine Learning}, pp.\  66--75.
  PMLR, 2021.

\bibitem[Agrawal \& Horel(2020)Agrawal and Horel]{agrawal2020optimal}
Rohit Agrawal and Thibaut Horel.
\newblock Optimal bounds between f-divergences and integral probability
  metrics.
\newblock In \emph{International Conference on Machine Learning}, pp.\
  115--124. PMLR, 2020.

\bibitem[Aminian et~al.(2022)Aminian, Abroshan, Khalili, Toni, and
  Rodrigues]{aminian2022information}
Gholamali Aminian, Mahed Abroshan, Mohammad~Mahdi Khalili, Laura Toni, and
  Miguel Rodrigues.
\newblock An information-theoretical approach to semi-supervised learning under
  covariate-shift.
\newblock In \emph{International Conference on Artificial Intelligence and
  Statistics}, pp.\  7433--7449. PMLR, 2022.

\bibitem[Bartlett \& Mendelson(2002)Bartlett and
  Mendelson]{bartlett2002rademacher}
Peter~L Bartlett and Shahar Mendelson.
\newblock Rademacher and gaussian complexities: Risk bounds and structural
  results.
\newblock \emph{Journal of Machine Learning Research}, 3\penalty0
  (Nov):\penalty0 463--482, 2002.

\bibitem[Ben-David et~al.(2006)Ben-David, Blitzer, Crammer, and
  Pereira]{ben2006analysis}
Shai Ben-David, John Blitzer, Koby Crammer, and Fernando Pereira.
\newblock Analysis of representations for domain adaptation.
\newblock \emph{Advances in neural information processing systems}, 19, 2006.

\bibitem[Ben-David et~al.(2010)Ben-David, Blitzer, Crammer, Kulesza, Pereira,
  and Vaughan]{ben2010theory}
Shai Ben-David, John Blitzer, Koby Crammer, Alex Kulesza, Fernando Pereira, and
  Jennifer~Wortman Vaughan.
\newblock A theory of learning from different domains.
\newblock \emph{Machine Learning}, 79\penalty0 (1-2):\penalty0 151--175, 2010.

\bibitem[Bretagnolle \& Huber(1979)Bretagnolle and
  Huber]{bretagnolle1979estimation}
Jean Bretagnolle and Catherine Huber.
\newblock Estimation des densit{\'e}s: risque minimax.
\newblock \emph{Zeitschrift f{\"u}r Wahrscheinlichkeitstheorie und verwandte
  Gebiete}, 47\penalty0 (2):\penalty0 119--137, 1979.

\bibitem[Bu et~al.(2019)Bu, Zou, and Veeravalli]{bu2019tightening}
Yuheng Bu, Shaofeng Zou, and Venugopal~V Veeravalli.
\newblock Tightening mutual information based bounds on generalization error.
\newblock In \emph{2019 IEEE International Symposium on Information Theory
  (ISIT)}, pp.\  587--591. IEEE, 2019.

\bibitem[Bu et~al.(2022)Bu, Aminian, Toni, Wornell, and
  Rodrigues]{bu2022characterizing}
Yuheng Bu, Gholamali Aminian, Laura Toni, Gregory~W Wornell, and Miguel
  Rodrigues.
\newblock Characterizing and understanding the generalization error of transfer
  learning with gibbs algorithm.
\newblock In \emph{International Conference on Artificial Intelligence and
  Statistics}, pp.\  8673--8699. PMLR, 2022.

\bibitem[Canonne(2022)]{canonne2022short}
Cl{\'e}ment~L Canonne.
\newblock A short note on an inequality between kl and tv.
\newblock \emph{arXiv preprint arXiv:2202.07198}, 2022.

\bibitem[Chen et~al.(2021)Chen, Shui, and Marchand]{chen2021generalization}
Qi~Chen, Changjian Shui, and Mario Marchand.
\newblock Generalization bounds for meta-learning: An information-theoretic
  analysis.
\newblock \emph{Advances in Neural Information Processing Systems}, 34, 2021.

\bibitem[Cortes et~al.(2019)Cortes, Greenberg, and Mohri]{cortes2019relative}
Corinna Cortes, Spencer Greenberg, and Mehryar Mohri.
\newblock Relative deviation learning bounds and generalization with unbounded
  loss functions.
\newblock \emph{Annals of Mathematics and Artificial Intelligence}, 85\penalty0
  (1):\penalty0 45--70, 2019.

\bibitem[Cover \& Thomas(2006)Cover and Thomas]{thomas2006elements}
Thomas~M. Cover and Joy~A. Thomas.
\newblock \emph{Elements of Information Theory (Wiley Series in
  Telecommunications and Signal Processing)}.
\newblock Wiley-Interscience, USA, 2006.
\newblock ISBN 0471241954.

\bibitem[Crammer et~al.(2008)Crammer, Kearns, and Wortman]{crammer2008learning}
Koby Crammer, Michael Kearns, and Jennifer Wortman.
\newblock Learning from multiple sources.
\newblock \emph{Journal of Machine Learning Research}, 9\penalty0 (8), 2008.

\bibitem[Crooks(2008)]{Crooks2008InequalitiesBT}
Gavin~E. Crooks.
\newblock Inequalities between the jenson-shannon and jeffreys divergences.
\newblock In \emph{Tech. Note 004}, 2008.

\bibitem[Cédric(2008)]{Villani2008}
Villani Cédric.
\newblock \emph{Optimal Transport: Old and New (Grundlehren der mathematischen
  Wissenschaften, 338)}.
\newblock Springer, 2008.

\bibitem[David et~al.(2010)David, Lu, Luu, and P{\'a}l]{david2010impossibility}
Shai~Ben David, Tyler Lu, Teresa Luu, and D{\'a}vid P{\'a}l.
\newblock Impossibility theorems for domain adaptation.
\newblock In \emph{Proceedings of the Thirteenth International Conference on
  Artificial Intelligence and Statistics}, pp.\  129--136. JMLR Workshop and
  Conference Proceedings, 2010.

\bibitem[Ganin et~al.(2016)Ganin, Ustinova, Ajakan, Germain, Larochelle,
  Laviolette, Marchand, and Lempitsky]{ganin2016domain}
Yaroslav Ganin, Evgeniya Ustinova, Hana Ajakan, Pascal Germain, Hugo
  Larochelle, Fran{\c{c}}ois Laviolette, Mario Marchand, and Victor Lempitsky.
\newblock Domain-adversarial training of neural networks.
\newblock \emph{The journal of machine learning research}, 17\penalty0
  (1):\penalty0 2096--2030, 2016.

\bibitem[Geiping et~al.(2022)Geiping, Goldblum, Pope, Moeller, and
  Goldstein]{geiping2022stochastic}
Jonas Geiping, Micah Goldblum, Phil Pope, Michael Moeller, and Tom Goldstein.
\newblock Stochastic training is not necessary for generalization.
\newblock In \emph{International Conference on Learning Representations}, 2022.

\bibitem[Germain et~al.(2020)Germain, Habrard, Laviolette, and
  Morvant]{germain2020pac}
Pascal Germain, Amaury Habrard, Fran{\c{c}}ois Laviolette, and Emilie Morvant.
\newblock Pac-bayes and domain adaptation.
\newblock \emph{Neurocomputing}, 379:\penalty0 379--397, 2020.

\bibitem[Gulrajani \& Lopez-Paz(2021)Gulrajani and Lopez-Paz]{gulrajani2021in}
Ishaan Gulrajani and David Lopez-Paz.
\newblock In search of lost domain generalization.
\newblock In \emph{International Conference on Learning Representations}, 2021.

\bibitem[Gulrajani et~al.(2017)Gulrajani, Ahmed, Arjovsky, Dumoulin, and
  Courville]{gulrajani2017improved}
Ishaan Gulrajani, Faruk Ahmed, Martin Arjovsky, Vincent Dumoulin, and Aaron~C
  Courville.
\newblock Improved training of wasserstein gans.
\newblock \emph{Advances in neural information processing systems}, 30, 2017.

\bibitem[Haghifam et~al.(2020)Haghifam, Negrea, Khisti, Roy, and
  Dziugaite]{haghifam2020sharpened}
Mahdi Haghifam, Jeffrey Negrea, Ashish Khisti, Daniel~M Roy, and
  Gintare~Karolina Dziugaite.
\newblock Sharpened generalization bounds based on conditional mutual
  information and an application to noisy, iterative algorithms.
\newblock \emph{Advances in Neural Information Processing Systems}, 2020.

\bibitem[Harutyunyan et~al.(2020)Harutyunyan, Reing, Ver~Steeg, and
  Galstyan]{harutyunyan2020improving}
Hrayr Harutyunyan, Kyle Reing, Greg Ver~Steeg, and Aram Galstyan.
\newblock Improving generalization by controlling label-noise information in
  neural network weights.
\newblock In \emph{International Conference on Machine Learning}, pp.\
  4071--4081. PMLR, 2020.

\bibitem[He et~al.(2021)He, Yan, and Tan]{he2021information}
Haiyun He, Hanshu Yan, and Vincent~YF Tan.
\newblock Information-theoretic generalization bounds for iterative
  semi-supervised learning.
\newblock \emph{arXiv preprint arXiv:2110.00926}, 2021.

\bibitem[Hull(1994)]{hull1994database}
Jonathan~J. Hull.
\newblock A database for handwritten text recognition research.
\newblock \emph{IEEE Transactions on pattern analysis and machine
  intelligence}, 16\penalty0 (5):\penalty0 550--554, 1994.

\bibitem[Jastrzebski et~al.(2021)Jastrzebski, Arpit, Astrand, Kerg, Wang,
  Xiong, Socher, Cho, and Geras]{jastrzebski2021catastrophic}
Stanislaw Jastrzebski, Devansh Arpit, Oliver Astrand, Giancarlo~B Kerg, Huan
  Wang, Caiming Xiong, Richard Socher, Kyunghyun Cho, and Krzysztof~J Geras.
\newblock Catastrophic fisher explosion: Early phase fisher matrix impacts
  generalization.
\newblock In \emph{International Conference on Machine Learning}. PMLR, 2021.

\bibitem[Jeffreys(1946)]{jeffreys1946invariant}
Harold Jeffreys.
\newblock An invariant form for the prior probability in estimation problems.
\newblock \emph{Proceedings of the Royal Society of London. Series A.
  Mathematical and Physical Sciences}, 186\penalty0 (1007):\penalty0 453--461,
  1946.

\bibitem[Jiao et~al.(2017)Jiao, Han, and Weissman]{jiao2017dependence}
Jiantao Jiao, Yanjun Han, and Tsachy Weissman.
\newblock Dependence measures bounding the exploration bias for general
  measurements.
\newblock In \emph{2017 IEEE International Symposium on Information Theory
  (ISIT)}, pp.\  1475--1479. IEEE, 2017.

\bibitem[Jose \& Simeone(2021{\natexlab{a}})Jose and
  Simeone]{jose2021information}
Sharu~Theresa Jose and Osvaldo Simeone.
\newblock Information-theoretic generalization bounds for meta-learning and
  applications.
\newblock \emph{Entropy}, 23\penalty0 (1):\penalty0 126, 2021{\natexlab{a}}.

\bibitem[Jose \& Simeone(2021{\natexlab{b}})Jose and
  Simeone]{jose2021informationJS}
Sharu~Theresa Jose and Osvaldo Simeone.
\newblock Information-theoretic bounds on transfer generalization gap based on
  jensen-shannon divergence.
\newblock In \emph{2021 29th European Signal Processing Conference (EUSIPCO)},
  pp.\  1461--1465. IEEE, 2021{\natexlab{b}}.

\bibitem[Jose et~al.(2021)Jose, Simeone, and Durisi]{jose2021transfer}
Sharu~Theresa Jose, Osvaldo Simeone, and Giuseppe Durisi.
\newblock Transfer meta-learning: Information-theoretic bounds and information
  meta-risk minimization.
\newblock \emph{IEEE Transactions on Information Theory}, 68\penalty0
  (1):\penalty0 474--501, 2021.

\bibitem[Kingma \& Welling(2013)Kingma and Welling]{kingma2013auto}
Diederik~P Kingma and Max Welling.
\newblock Auto-encoding variational bayes.
\newblock \emph{arXiv preprint arXiv:1312.6114}, 2013.

\bibitem[LeCun et~al.(2010)LeCun, Cortes, and Burges]{lecun2010mnist}
Yann LeCun, Corinna Cortes, and CJ~Burges.
\newblock Mnist handwritten digit database.
\newblock \emph{ATT Labs [Online]. Available:
  http://yann.lecun.com/exdb/mnist}, 2, 2010.

\bibitem[Li et~al.(2018)Li, Pan, Wang, and Kot]{li2018domain}
Haoliang Li, Sinno~Jialin Pan, Shiqi Wang, and Alex~C Kot.
\newblock Domain generalization with adversarial feature learning.
\newblock In \emph{Proceedings of the IEEE conference on computer vision and
  pattern recognition}, pp.\  5400--5409, 2018.

\bibitem[Liang et~al.(2020)Liang, Hu, and Feng]{liang2020we}
Jian Liang, Dapeng Hu, and Jiashi Feng.
\newblock Do we really need to access the source data? source hypothesis
  transfer for unsupervised domain adaptation.
\newblock In \emph{International Conference on Machine Learning}, pp.\
  6028--6039. PMLR, 2020.

\bibitem[Mansour et~al.(2009)Mansour, Mohri, and Rostamizadeh]{MansourMR09}
Yishay Mansour, Mehryar Mohri, and Afshin Rostamizadeh.
\newblock Domain adaptation: Learning bounds and algorithms.
\newblock In \emph{The 22nd Conference on Learning Theory}, 2009.

\bibitem[Masiha et~al.(2021)Masiha, Gohari, Yassaee, and
  Aref]{masiha2021learning}
Mohammad~Saeed Masiha, Amin Gohari, Mohammad~Hossein Yassaee, and Mohammad~Reza
  Aref.
\newblock Learning under distribution mismatch and model misspecification.
\newblock In \emph{2021 IEEE International Symposium on Information Theory
  (ISIT)}, pp.\  2912--2917. IEEE, 2021.

\bibitem[Mohri et~al.(2018)Mohri, Rostamizadeh, and
  Talwalkar]{mohri2018foundations}
Mehryar Mohri, Afshin Rostamizadeh, and Ameet Talwalkar.
\newblock \emph{Foundations of machine learning}.
\newblock MIT press, 2018.

\bibitem[Negrea et~al.(2019)Negrea, Haghifam, Dziugaite, Khisti, and
  Roy]{negrea2019information}
Jeffrey Negrea, Mahdi Haghifam, Gintare~Karolina Dziugaite, Ashish Khisti, and
  Daniel~M Roy.
\newblock Information-theoretic generalization bounds for sgld via
  data-dependent estimates.
\newblock \emph{Advances in Neural Information Processing Systems}, 2019.

\bibitem[Netzer et~al.(2011)Netzer, Wang, Coates, Bissacco, Wu, and
  Ng]{netzer2011reading}
Yuval Netzer, Tao Wang, Adam Coates, Alessandro Bissacco, Bo~Wu, and Andrew~Y
  Ng.
\newblock Reading digits in natural images with unsupervised feature learning.
\newblock In \emph{NeurIPS Workshop on Deep Learning and Unsupervised Feature
  Learning}, 2011.

\bibitem[Neu et~al.(2021)Neu, Dziugaite, Haghifam, and Roy]{neu2021information}
Gergely Neu, Gintare~Karolina Dziugaite, Mahdi Haghifam, and Daniel~M Roy.
\newblock Information-theoretic generalization bounds for stochastic gradient
  descent.
\newblock In \emph{Conference on Learning Theory}. PMLR, 2021.

\bibitem[Nguyen et~al.(2022)Nguyen, Tran, Gal, Torr, and Baydin]{nguyen2022kl}
A.~Tuan Nguyen, Toan Tran, Yarin Gal, Philip Torr, and Atilim~Gunes Baydin.
\newblock {KL} guided domain adaptation.
\newblock In \emph{International Conference on Learning Representations}, 2022.

\bibitem[Nguyen et~al.(2010)Nguyen, Wainwright, and
  Jordan]{nguyen2010estimating}
XuanLong Nguyen, Martin~J Wainwright, and Michael~I Jordan.
\newblock Estimating divergence functionals and the likelihood ratio by convex
  risk minimization.
\newblock \emph{IEEE Transactions on Information Theory}, 56\penalty0
  (11):\penalty0 5847--5861, 2010.

\bibitem[Palomar \& Verd{\'u}(2008)Palomar and Verd{\'u}]{palomar2008lautum}
Daniel~P Palomar and Sergio Verd{\'u}.
\newblock Lautum information.
\newblock \emph{IEEE transactions on information theory}, 54\penalty0
  (3):\penalty0 964--975, 2008.

\bibitem[Paszke et~al.(2019)Paszke, Gross, Massa, Lerer, Bradbury, Chanan,
  Killeen, Lin, Gimelshein, Antiga, et~al.]{paszke2019pytorch}
Adam Paszke, Sam Gross, Francisco Massa, Adam Lerer, James Bradbury, Gregory
  Chanan, Trevor Killeen, Zeming Lin, Natalia Gimelshein, Luca Antiga, et~al.
\newblock Pytorch: An imperative style, high-performance deep learning library.
\newblock \emph{Advances in Neural Information Processing Systems},
  32:\penalty0 8026--8037, 2019.

\bibitem[Peng et~al.(2017)Peng, Usman, Kaushik, Hoffman, Wang, and
  Saenko]{peng2017visda}
Xingchao Peng, Ben Usman, Neela Kaushik, Judy Hoffman, Dequan Wang, and Kate
  Saenko.
\newblock Visda: The visual domain adaptation challenge.
\newblock \emph{arXiv preprint arXiv:1710.06924}, 2017.

\bibitem[Pensia et~al.(2018)Pensia, Jog, and Loh]{pensia2018generalization}
Ankit Pensia, Varun Jog, and Po-Ling Loh.
\newblock Generalization error bounds for noisy, iterative algorithms.
\newblock In \emph{2018 IEEE International Symposium on Information Theory
  (ISIT)}. IEEE, 2018.

\bibitem[Polyanskiy \& Wu(2019)Polyanskiy and Wu]{polyanskiy2019lecture}
Yury Polyanskiy and Yihong Wu.
\newblock Lecture notes on information theory.
\newblock \emph{Lecture Notes for 6.441 (MIT), ECE 563 (UIUC), STAT 364 (Yale),
  2019.}, 2019.

\bibitem[Redko et~al.(2020)Redko, Morvant, Habrard, Sebban, and
  Bennani]{redko2020survey}
Ievgen Redko, Emilie Morvant, Amaury Habrard, Marc Sebban, and Younes Bennani.
\newblock A survey on domain adaptation theory.
\newblock \emph{arXiv preprint arXiv:2004.11829}, 2020.

\bibitem[Rezazadeh et~al.(2021)Rezazadeh, Jose, Durisi, and
  Simeone]{rezazadeh2021conditional}
Arezou Rezazadeh, Sharu~Theresa Jose, Giuseppe Durisi, and Osvaldo Simeone.
\newblock Conditional mutual information-based generalization bound for meta
  learning.
\newblock In \emph{2021 IEEE International Symposium on Information Theory
  (ISIT)}, pp.\  1176--1181. IEEE, 2021.

\bibitem[Rodr{\'\i}guez-G{\'a}lvez et~al.(2021)Rodr{\'\i}guez-G{\'a}lvez,
  Bassi, Thobaben, and Skoglund]{rodriguez2021random}
Borja Rodr{\'\i}guez-G{\'a}lvez, Germ{\'a}n Bassi, Ragnar Thobaben, and Mikael
  Skoglund.
\newblock On random subset generalization error bounds and the stochastic
  gradient langevin dynamics algorithm.
\newblock In \emph{2020 IEEE Information Theory Workshop (ITW)}, pp.\  1--5.
  IEEE, 2021.

\bibitem[Rodr{\'\i}guez~G{\'a}lvez et~al.(2021)Rodr{\'\i}guez~G{\'a}lvez,
  Bassi, Thobaben, and Skoglund]{rodriguez2021tighter}
Borja Rodr{\'\i}guez~G{\'a}lvez, Germ{\'a}n Bassi, Ragnar Thobaben, and Mikael
  Skoglund.
\newblock Tighter expected generalization error bounds via wasserstein
  distance.
\newblock \emph{Advances in Neural Information Processing Systems}, 34, 2021.

\bibitem[Russo \& Zou(2016)Russo and Zou]{russo2016controlling}
Daniel Russo and James Zou.
\newblock Controlling bias in adaptive data analysis using information theory.
\newblock In \emph{Artificial Intelligence and Statistics}. PMLR, 2016.

\bibitem[Russo \& Zou(2019)Russo and Zou]{russo2019much}
Daniel Russo and James Zou.
\newblock How much does your data exploration overfit? controlling bias via
  information usage.
\newblock \emph{IEEE Transactions on Information Theory}, 66\penalty0
  (1):\penalty0 302--323, 2019.

\bibitem[Shen et~al.(2018)Shen, Qu, Zhang, and Yu]{shen2018wasserstein}
Jian Shen, Yanru Qu, Weinan Zhang, and Yong Yu.
\newblock Wasserstein distance guided representation learning for domain
  adaptation.
\newblock In \emph{Thirty-second AAAI conference on artificial intelligence},
  2018.

\bibitem[Shen et~al.(2022)Shen, Bu, and Wornell]{shen2022benefits}
Maohao Shen, Yuheng Bu, and Gregory Wornell.
\newblock On the benefits of selectivity in pseudo-labeling for unsupervised
  multi-source-free domain adaptation.
\newblock \emph{arXiv preprint arXiv:2202.00796}, 2022.

\bibitem[Shui et~al.(2020)Shui, Chen, Wen, Zhou, Gagn{\'e}, and
  Wang]{shui2020beyond}
Changjian Shui, Qi~Chen, Jun Wen, Fan Zhou, Christian Gagn{\'e}, and Boyu Wang.
\newblock Beyond $\mathcal{H}$-divergence: Domain adaptation theory with
  jensen-shannon divergence.
\newblock \emph{arXiv preprint arXiv:2007.15567}, 2020.

\bibitem[Smith et~al.(2021)Smith, Dherin, Barrett, and De]{smith2020origin}
Samuel~L Smith, Benoit Dherin, David Barrett, and Soham De.
\newblock On the origin of implicit regularization in stochastic gradient
  descent.
\newblock In \emph{International Conference on Learning Representations}, 2021.

\bibitem[Steinke \& Zakynthinou(2020)Steinke and
  Zakynthinou]{steinke2020reasoning}
Thomas Steinke and Lydia Zakynthinou.
\newblock Reasoning about generalization via conditional mutual information.
\newblock In \emph{Conference on Learning Theory}. PMLR, 2020.

\bibitem[Sun \& Saenko(2016)Sun and Saenko]{sun2016deep}
Baochen Sun and Kate Saenko.
\newblock Deep coral: Correlation alignment for deep domain adaptation.
\newblock In \emph{European conference on computer vision}, pp.\  443--450.
  Springer, 2016.

\bibitem[Van~der Maaten \& Hinton(2008)Van~der Maaten and
  Hinton]{van2008visualizing}
Laurens Van~der Maaten and Geoffrey Hinton.
\newblock Visualizing data using t-sne.
\newblock \emph{Journal of machine learning research}, 9\penalty0 (11), 2008.

\bibitem[Vapnik(1998)]{SLT98Vapnik}
Vladimir Vapnik.
\newblock \emph{Statistical learning theory}.
\newblock Wiley, 1998.
\newblock ISBN 978-0-471-03003-4.

\bibitem[Wang et~al.(2019)Wang, Mendez, Cai, and Eaton]{wang2019transfer}
Boyu Wang, Jorge Mendez, Mingbo Cai, and Eric Eaton.
\newblock Transfer learning via minimizing the performance gap between domains.
\newblock \emph{Advances in Neural Information Processing Systems}, 32, 2019.

\bibitem[Wang et~al.(2021{\natexlab{a}})Wang, Gao, and
  Calmon]{wang2021generalization}
Hao Wang, Rui Gao, and Flavio~P Calmon.
\newblock Generalization bounds for noisy iterative algorithms using properties
  of additive noise channels.
\newblock \emph{arXiv preprint arXiv:2102.02976}, 2021{\natexlab{a}}.

\bibitem[Wang et~al.(2021{\natexlab{b}})Wang, Lan, Liu, Ouyang, Zeng, and
  Qin]{wang2021generalizing}
Jindong Wang, Cuiling Lan, Chang Liu, Yidong Ouyang, Wenjun Zeng, and Tao Qin.
\newblock Generalizing to unseen domains: A survey on domain generalization.
\newblock \emph{arXiv preprint arXiv:2103.03097}, 2021{\natexlab{b}}.

\bibitem[Wang \& Mao(2022{\natexlab{a}})Wang and Mao]{wang2022generalization}
Ziqiao Wang and Yongyi Mao.
\newblock On the generalization of models trained with {SGD}:
  Information-theoretic bounds and implications.
\newblock In \emph{International Conference on Learning Representations},
  2022{\natexlab{a}}.

\bibitem[Wang \& Mao(2022{\natexlab{b}})Wang and Mao]{wang2022two}
Ziqiao Wang and Yongyi Mao.
\newblock Two facets of sde under an information-theoretic lens: Generalization
  of sgd via training trajectories and via terminal states.
\newblock \emph{arXiv preprint arXiv:2211.10691}, 2022{\natexlab{b}}.

\bibitem[Wang \& Mao(2023)Wang and Mao]{wang2023tighter}
Ziqiao Wang and Yongyi Mao.
\newblock Tighter information-theoretic generalization bounds from
  supersamples.
\newblock \emph{arXiv preprint arXiv:2302.02432}, 2023.

\bibitem[Wilson \& Cook(2020)Wilson and Cook]{wilson2020survey}
Garrett Wilson and Diane~J Cook.
\newblock A survey of unsupervised deep domain adaptation.
\newblock \emph{ACM Transactions on Intelligent Systems and Technology (TIST)},
  11\penalty0 (5):\penalty0 1--46, 2020.

\bibitem[Wu et~al.(2020)Wu, Manton, Aickelin, and Zhu]{wu2020information}
Xuetong Wu, Jonathan~H Manton, Uwe Aickelin, and Jingge Zhu.
\newblock Information-theoretic analysis for transfer learning.
\newblock In \emph{2020 IEEE International Symposium on Information Theory
  (ISIT)}, pp.\  2819--2824. IEEE, 2020.

\bibitem[Xu \& Raginsky(2017)Xu and Raginsky]{xu2017information}
Aolin Xu and Maxim Raginsky.
\newblock Information-theoretic analysis of generalization capability of
  learning algorithms.
\newblock \emph{Advances in Neural Information Processing Systems}, 2017.

\bibitem[Yu et~al.(2022)Yu, Yang, Wei, Ma, and Steinhardt]{yu2022predicting}
Yaodong Yu, Zitong Yang, Alexander Wei, Yi~Ma, and Jacob Steinhardt.
\newblock Predicting out-of-distribution error with the projection norm.
\newblock \emph{arXiv preprint arXiv:2202.05834}, 2022.

\bibitem[Zhang et~al.(2019)Zhang, Liu, Long, and Jordan]{zhang2019bridging}
Yuchen Zhang, Tianle Liu, Mingsheng Long, and Michael Jordan.
\newblock Bridging theory and algorithm for domain adaptation.
\newblock In \emph{International Conference on Machine Learning}, pp.\
  7404--7413. PMLR, 2019.

\bibitem[Zhao et~al.(2019)Zhao, Des~Combes, Zhang, and
  Gordon]{zhao2019learning}
Han Zhao, Remi~Tachet Des~Combes, Kun Zhang, and Geoffrey Gordon.
\newblock On learning invariant representations for domain adaptation.
\newblock In \emph{International Conference on Machine Learning}, pp.\
  7523--7532. PMLR, 2019.

\bibitem[Zhou et~al.(2021)Zhou, Liu, Qiao, Xiang, and
  Change~Loy]{zhou2021domain}
Kaiyang Zhou, Ziwei Liu, Yu~Qiao, Tao Xiang, and Chen Change~Loy.
\newblock Domain generalization: A survey.
\newblock \emph{arXiv e-prints}, pp.\  arXiv--2103, 2021.

\end{thebibliography}
